\documentclass[11pt]{article}

\usepackage[T1]{fontenc}
\usepackage[utf8]{inputenc}
\usepackage[margin=1in]{geometry}
\usepackage{setspace}
\usepackage[sc,osf]{mathpazo}

\usepackage{amsmath}
\usepackage{amssymb}
\usepackage{amsfonts}

\usepackage{amsthm}
\usepackage{mathtools}
\usepackage{bm}

\usepackage{graphicx}
\usepackage{subcaption}
\usepackage{booktabs}
\usepackage{tabularx}
\usepackage{multirow}
\usepackage{array}
\usepackage{float}
\usepackage{enumitem}
\usepackage[dvipsnames]{xcolor}
\usepackage{tikz}
\usetikzlibrary{arrows.meta,positioning,decorations.markings,patterns,calc}

\usepackage{authblk}

\usepackage[authoryear,round]{natbib}
\usepackage[colorlinks=true,allcolors=green!60!black]{hyperref}
\usepackage{microtype}

\newcommand{\Pbb}{\mathbb{P}}
\newcommand{\Ebb}{\mathbb{E}}
\newcommand{\Rbb}{\mathbb{R}}
\newcommand{\cH}{\mathcal{H}}
\newcommand{\cA}{\mathcal{A}}

\newcommand{\cC}{\mathcal{C}}
\newcommand{\cD}{\mathcal{D}}

\newcommand{\cF}{\mathcal{F}}
\newcommand{\cG}{\mathcal{G}}
\newcommand{\cP}{\mathcal{P}}
\newcommand{\cR}{\mathcal{R}}
\newcommand{\cX}{\mathcal{X}}
\newcommand{\cY}{\mathcal{Y}}

\newcommand{\dx}{\,\mathrm{d}}
\newcommand{\KL}{D_{\mathrm{KL}}}

\theoremstyle{plain}
\newtheorem{theorem}{Theorem}[section]
\newtheorem{proposition}[theorem]{Proposition}
\newtheorem{lemma}[theorem]{Lemma}
\newtheorem{corollary}[theorem]{Corollary}

\theoremstyle{definition}
\newtheorem{definition}[theorem]{Definition}
\newtheorem{example}[theorem]{Example}

\theoremstyle{remark}
\newtheorem{remark}[theorem]{Remark}

\title{\textbf{Bayes with No Shame:\\[4pt]
Admissibility Geometries of Predictive Inference}}

\author[1]{Nicholas~G.\ Polson}
\author[2]{Daniel Zantedeschi}

\affil[1]{Booth School of Business, University of Chicago, Chicago, IL 60637\\
  \texttt{ngp@chicagobooth.edu}}
\affil[2]{Muma College of Business, University of South Florida, Tampa, FL 33620\\
  \texttt{danielz@usf.edu}}

\date{August 2026}

\hypersetup{
  pdftitle={Bayes with No Shame: Admissibility Geometries of Predictive Inference},
  pdfauthor={Nicholas G. Polson and Daniel Zantedeschi},
  pdfsubject={Predictive inference and criterion-relative admissibility},
  pdfkeywords={admissibility, Blackwell dominance, martingales, e-processes,
    conformal prediction, sequential inference}
}

\begin{document}
\maketitle

\begin{abstract}\noindent
Predictive systems may combine a predictor, a sequential monitor,
a prediction set, and an online strategy, each governed by a
different optimality criterion.  We study when a guarantee for one
component can be transferred to another.
Four admissibility geometries shape sequential and distribution-free
inference: Blackwell risk dominance over convex risk sets,
anytime-valid admissibility within the $e$-process class,
fixed-level marginal coverage with expected-length efficiency
within a declared rank-indexed family under exchangeability, and choice-based
approachability (CApp) boundary-feasibility, under which each
parameter's time-averaged risk reaches its oracle boundary value
in Ces\`{a}ro average.
Each geometry carries a distinct certificate of optimality: a
supporting-hyperplane prior, a nonnegative martingale in the
point-null AV setting, a conformal exchangeability rank, and a Ces\`{a}ro steering
argument, respectively.  We embed all four in a common ambient
space, the space $\Sigma$ of \emph{predictive systems}: tuples
$\Delta=(\delta,E,\hat C,\sigma)$ of predictor, $e$-process,
prediction set, and online strategy.  Within $\Sigma$ we prove a
\emph{witness-based non-nesting theorem}: for each ordered pair of
the four criterion classes, an explicit predictive system,
active in both relevant coordinates, lies in one class but not
the other (Theorems~\ref{thm:separation} and~\ref{thm:extended-separation}).
The result establishes non-nesting across different object spaces
and partial orders; it does not claim that the four criteria are
practically incompatible.

Bayesian posterior predictive means under a single prior are
martingales under their prior predictive law.  For a point null,
anytime-valid admissibility within $e$-processes is equivalent to
the nonnegative martingale property \citep{ramdas2022}.  By
contrast, self-consistency under a predictor's own law does not
imply Blackwell admissibility, as shown by the Bernoulli log-loss
counterexample (Theorem~\ref{thm:insuff}); neither coverage validity
nor CApp boundary-feasibility requires a martingale structure.

All four criteria fit a common design template: specify the
criterion's feasibility constraint, then optimize Bayesian
integrated risk within it.  Their decision spaces, partial orders,
and risk functionals nevertheless remain distinct.  Admissibility
is criterion-relative.
\end{abstract}



\section{Introduction}\label{sec:intro}

A classical example in Blackwell-style statistical decision
theory \citep{blackwell1954} concerns prediction of the next outcome
of a Bernoulli sequence under log loss.  The natural plug-in rule
$\hat p_n=S_n/n$ looks beyond reproach: it is a martingale under
its own predictive law, it concentrates on the true~$\theta$ by
the law of large numbers, and it is the maximum likelihood
estimator at every sample size.  Yet for every $\theta\in(0,1)$
and every $n\ge 1$ it is strictly dominated by the Bayes
predictive $\hat p_n^B=(S_n+\tfrac12)/(n+1)$.  The plug-in
assigns probability zero to events that occur with positive
probability, incurring infinite log loss on the event
$\{S_n\in\{0,n\}\}$ and hence infinite risk; its martingale
coherence does not rescue it
from inadmissibility.  Blackwell's point was geometric: the
plug-in risk vector sits above the lower boundary of the convex
risk set, and a supporting-hyperplane prior, the Jeffreys
$\mathrm{Beta}(\tfrac12,\tfrac12)$, exposes a strictly
dominating Bayes alternative on that boundary.

The same distinction appears in modern predictive systems.
Bayesian prediction is evaluated through proper scoring rules and
posterior predictives; conformal prediction
\citep{vovk2005} achieves distribution-free coverage without
optimizing any loss; $e$-processes and safe anytime-valid testing
\citep{shafer2019,ramdas2023,grunwald2024} control type-I error at every
stopping time through the $e$-process condition; for point-null
admissibility, the corresponding condition is the nonnegative
martingale property; and online calibration and defensive forecasting
\citep{foster1998, vovk2005defensive} achieve calibration in the
Ces\`{a}ro sense through fixed-point arguments, without optimizing
any per-round loss function.  These procedures answer different
questions, and a guarantee established under one criterion need
not hold under another.

\paragraph{Practical motivation: modular predictive systems.}
A foundation model may be trained under a
proper scoring rule, monitored through sequential safety indicators,
wrapped by a conformal uncertainty layer, and evaluated for
calibration after deployment.  The relevant questions are which
criterion supports each optimality claim and which guarantees fail
to transfer between modules.

We use ``no-shame'' throughout as a mnemonic for
Blackwell-admissible, and the term is chosen deliberately.
\citet{williams1993shame} distinguishes guilt, which answers to an
external rule, from shame, which is the recognition of falling
below a standard one endorses oneself.  Blackwell dominance is
shame in precisely this sense: a dominated rule is condemned by no
outside authority, only by the loss function the analyst
declared.  The metaphor does rhetorical work because it locates
the criticism \emph{inside} the criterion; for the same reason
it points to pluralism rather than to a ranking.  A verdict
that is internal to a standard offers no standard-free court of
appeal: the analyst who endorses anytime validity, or marginal
coverage, or long-run calibration, answers to that standard and
not automatically to the others.  The paper's formal content is
order-theoretic; the metaphor plays no role in the results, and
readers may substitute ``Blackwell-admissible'' wherever
``no-shame'' appears.  The separation theorems
(Theorems~\ref{thm:separation} and~\ref{thm:extended-separation})
formalize the same point: admissibility is criterion-relative.

The four criteria, each defined relative to a different
performance objective and a different certificate of optimality,
are as follows.  \emph{Blackwell admissibility}
(Sections~\ref{sec:primitives}--\ref{sec:geometry}): no competing
rule has uniformly lower risk over~$\Theta$; its certificate is a
supporting-hyperplane prior.  \emph{Anytime-valid admissibility}
(Section~\ref{sec:separation}): for a point null, the e-process is a
nonnegative martingale under the null, equivalent to admissibility
within $\cC_{\mathrm{AV}}$; its certificate is a nonnegative martingale.
\emph{Marginal coverage validity} (Section~\ref{sec:separation}):
$\Pbb(Y_{n+1}\in\hat C_n(X_{n+1}))\ge 1-\alpha$ under
exchangeability, with expected-length efficiency within the
fixed-level feasible class; its certificate is a conformal
exchangeability rank.  \emph{CApp boundary-feasibility}
(Section~\ref{sec:constructive}): the time-averaged risk converges
to the oracle boundary value $\partial_-\cR(\theta)$ for every
$\theta$;
its certificate is a Ces\`{a}ro steering argument.

Our contributions are as follows.

\begin{enumerate}[label=\textup{(\arabic*)},leftmargin=*]
\item \emph{Witness-based non-nesting in a common space.}  We
introduce the space $\Sigma$ of predictive systems
$\Delta=(\delta,E,\hat C,\sigma)$ (Definition~\ref{def:predictive-system})
as a common ambient object for the four criteria, and prove that
the four corresponding admissibility classes $\mathfrak{B}$,
$\mathfrak{A}$, $\mathfrak{C}$, $\mathfrak{D}\subseteq\Sigma$ are
pairwise non-nested by exhibiting, for every ordered pair of
classes, an explicit predictive system that is active in both
relevant coordinates and lies in one class but not the other
(Theorems~\ref{thm:separation} and~\ref{thm:extended-separation}).
This product-space formulation makes the assignment of guarantees
to modules explicit.
\item \emph{Distinguished martingale claims.}  We separate three
notions of coherence (Definition~\ref{def:three-coherence}):
fixed-sample Blackwell admissibility, single-prior sequential
Bayes coherence, and self-consistency under a data-dependent
predictive law.  Single-prior Bayes implies martingale coherence
under the prior predictive (Proposition~\ref{prop:bayes-martingale}),
and, for the point-null setting used in the separation results,
the nonnegative martingale property within $\cC_{\mathrm{AV}}$
is equivalent to admissibility \citep{ramdas2022}.  The Bernoulli
log-loss counterexample (Theorem~\ref{thm:insuff}) supplies a
self-consistent predictor that is not Blackwell admissible; this
witness, not a general theorem, establishes that self-consistency
is not sufficient for Blackwell admissibility.
\item \emph{Constrained Bayes as a design template.}  For each
criterion we specify a decision space $\cD_C$, feasibility set
$\cF_C\subseteq\cD_C$, and risk functional $\cR_C$, then minimize
Bayesian integrated risk over $\cF_C$
(Definition~\ref{def:cb}).  The common form does not identify the
four decision spaces or their partial orders.
\end{enumerate}

Because the four procedures have different types, their
non-nesting may appear definitional.  The separation theorems use
the same Bernoulli experiment and active coordinates to show that
none of the criterion orders subsumes another.

\paragraph{Roadmap.}
Sections~\ref{sec:primitives}--\ref{sec:geometry} develop the
Blackwell geometry on the convex risk set, including the
supporting-hyperplane identification and the full-support
Bayes-implies-no-shame corollaries.
Section~\ref{sec:martingale} introduces the martingale layer and
the three coherence notions, and establishes the plug-in
counterexample.  Section~\ref{sec:separation} introduces the
predictive-system meta-space $\Sigma$, defines fixed-level
coverage admissibility, presents the constrained Bayes schema, and
proves the first separation theorem
(Theorem~\ref{thm:separation}).  Section~\ref{sec:constructive}
develops the constructive--Ces\`{a}ro distinction and the extended
separation (Theorem~\ref{thm:extended-separation}) including the
misspecified-prior witness.  Section~\ref{sec:bernoulli} collects
the Bernoulli laboratory and
Section~\ref{sec:numerical} reports three Monte Carlo
illustrations.  Section~\ref{sec:applications} reframes
forecasting, sequential testing, and conformal wrappers as
motivations rather than consequences;
Section~\ref{sec:conclusion} closes with the discussion.

\section{Primitive Objects}\label{sec:primitives}

The decision-theoretic framework requires five objects: a parameter
space, an action space, a loss function, a sample space, and a
statistical model.  We adopt the extended-real formulation that
allows $+\infty$ risk, accommodating proper scoring rules such as
log loss from the outset.

\begin{definition}[Statistical decision problem]\label{def:sdp}
A \emph{statistical decision problem} is a tuple
$(\Theta,\cA,L,\mathcal{X},\mathcal{P})$ where:
\begin{enumerate}[label=(\roman*),nosep]
\item $\Theta\subset\Rbb^d$ is the \emph{parameter space}, compact
  and metrizable.
\item $\cA\subset\Rbb^m$ is the \emph{action space}, compact and
  metrizable.
\item $\mathcal{X}$ is the \emph{sample space}, a Polish space, and
  $\mathcal{P}=\{P_\theta:\theta\in\Theta\}$ is the statistical model.
\item $L:\Theta\times\cA\to[0,\infty]$ is the \emph{loss function},
  satisfying:
  \begin{enumerate}[label=(\alph*),nosep]
  \item $L(\theta,a)$ is measurable in $\theta$ for every $a$;
  \item $L(\theta,\cdot)$ is lower semicontinuous for every $\theta$;
  \item $L$ is bounded below (by zero, without loss of generality).
  \end{enumerate}
\end{enumerate}
Risks are allowed to take the value $+\infty$; dominance is defined
in the extended-real sense via the coordinatewise ordering on
$[0,\infty]^\Theta$.
\end{definition}

\paragraph{Scope of compactness.}
Compactness of $\Theta$ is used for the geometric, complete-class,
and $\Sigma$-level results below.  The definitions of risk and
dominance, the measure-relative martingale statements, and CApp
boundary-feasibility extend verbatim to noncompact parameter spaces.
Results stated on the full Bernoulli model $\Theta=(0,1)$ are
standalone pointwise results and do not invoke the compactness-based
geometry; geometric and $\Sigma$-level conclusions are instantiated
on explicitly compact Bernoulli submodels.

\begin{definition}[Decision rules]\label{def:rules}
Given data $X^n=(X_1,\ldots,X_n)\in\mathcal{X}^n$, a
\emph{(randomized) decision rule} is a measurable map
$\delta:\mathcal{X}^n\to\Delta(\cA)$, where $\Delta(\cA)$ denotes
the probability measures on $\cA$.  The class $\cD$ of all decision
rules is convex: for $\lambda\in[0,1]$, the mixture
$\delta_\lambda=\lambda\delta_1+(1-\lambda)\delta_2$ is defined by
drawing from $\delta_1$ or $\delta_2$ with probabilities $\lambda$
and $1-\lambda$ independently of $X^n$.
\end{definition}

\paragraph{Why extended-real risk.}
Proper scoring rules \citep{gneiting2007} used in this paper,
chiefly the log
loss $L(\theta,p)=-\theta\log p-(1-\theta)\log(1-p)$, take the
value $+\infty$ at boundary predictions $p\in\{0,1\}$ when
$\theta\in(0,1)$.  Allowing $R(\theta,\delta)\in[0,\infty]$ from
the outset is therefore essential, not a technicality: the
plug-in counterexample of Section~\ref{sec:martingale} hinges
on a predictor whose loss is $+\infty$ with positive probability,
and hence whose risk is $+\infty$; truncating to bounded loss would erase the
phenomenon.

\begin{definition}[Risk function]\label{def:risk}
The \emph{risk function} of a (randomized) decision rule
$\delta\in\cD$ is
\[
  R(\theta,\delta)
  \;=\; \Ebb_\theta\!\left[\,\int_\cA L(\theta,a)\,\dx\delta(X^n)(a)\right]
  \;\in\; [0,\infty],
  \qquad \theta\in\Theta,
\]
where the inner integral marginalizes over the action $a$ drawn
from the conditional distribution $\delta(X^n)\in\Delta(\cA)$, and
the outer expectation is over $X^n\sim P_\theta$.  For
deterministic rules $\delta(X^n)=a(X^n)$ this reduces to
$R(\theta,\delta)=\Ebb_\theta[L(\theta,a(X^n))]$.
Bayes risks are well-defined in $[0,\infty]$.  Convexity follows
from randomization.  Lower semicontinuity under the topology used
in Section~\ref{sec:geometry} is part of the standing regularity
assumptions stated there, and follows by
direct verification in the finite-sample settings used for the
explicit witnesses.
\end{definition}

\begin{example}[Bernoulli log-loss, our running illustration]%
\label{ex:bern-running}
Throughout the paper, we return to the following Bernoulli
laboratory.  On compact submodels it is an instance of
Definitions~\ref{def:sdp}--\ref{def:risk}.  Let
$X_i\overset{\mathrm{iid}}{\sim}\mathrm{Bernoulli}(\theta)$ with
$\theta\in\Theta\subseteq(0,1)$.  The action space is
$\cA=[0,1]$ (a forecast probability).  The loss is the binary
log loss
\[
  L(\theta,p) = -\theta\log p - (1-\theta)\log(1-p),
  \qquad p\in[0,1],
\]
extended by $L(\theta,0)=L(\theta,1)=+\infty$ when $\theta\in(0,1)$.
For a deterministic predictor $\hat p_n=\hat p_n(X^n)$, the
risk is
\[
  R(\theta,\hat p_n) = \Ebb_\theta\bigl[L(\theta,\hat p_n(X^n))\bigr]
  \;=\; \Ebb_\theta\bigl[\KL\bigl(\mathrm{Bern}(\theta)\,\|\,
       \mathrm{Bern}(\hat p_n(X^n))\bigr)\bigr] + H(\theta),
\]
where $H(\theta)=-\theta\log\theta-(1-\theta)\log(1-\theta)$ is
the binary entropy and $\KL$ is the Kullback--Leibler divergence.
The entropy $H(\theta)$ is the \emph{oracle coordinatewise lower
envelope} under log loss: it is the risk that would be attained at
$\theta$ by the unattainable rule $\hat p_n\equiv\theta$.  The
excess risk of any data-based predictor is the expected
$\KL$ term in the display above, which is nonnegative and, for
the shrinkage and universal-coding predictors considered below,
vanishes asymptotically.  Fixed-sample Blackwell admissibility
(Definition~\ref{def:admissibility} below and
Section~\ref{sec:geometry}) is coordinatewise undominatedness of
the full risk vector across~$\Theta$, not pointwise attainment of
$H(\theta)$.  When compactness is needed for complete-class
arguments we work on compact Bernoulli submodels
$\Theta=[\eta,1-\eta]\subset(0,1)$; the plug-in pathology itself
holds pointwise for every $\theta\in(0,1)$.  In the
Ces\`{a}ro/log-loss examples of
Section~\ref{sec:constructive}, universal coding strategies such
as KT \citep{krichevsky1981} drive the average regret against the
best constant in hindsight to zero, so their time-averaged risk approaches the
oracle envelope $H(\theta)$ at rate $O((\log n)/n)$.  We use this
Bernoulli example as the running illustration for every notion in
the paper (predictor, $e$-process, prediction set, online strategy).
\end{example}

\begin{definition}[Dominance and admissibility]\label{def:admissibility}
The partial order on $[0,\infty]^\Theta$ is defined by
$r\le r'$ if and only if $r(\theta)\le r'(\theta)$ for all
$\theta\in\Theta$.  Rule $\delta'$ \emph{dominates} $\delta$ if
$R(\theta,\delta')\le R(\theta,\delta)$ for all $\theta\in\Theta$
with strict inequality for some $\theta_0\in\Theta$.  A rule $\delta$
is \emph{Blackwell admissible} if no rule in $\cD$ dominates it.
\end{definition}

\begin{definition}[Risk set]\label{def:riskset}
For $\Theta=\{\theta_1,\ldots,\theta_k\}$ finite, associate to each
$\delta\in\cD$ its \emph{risk vector}
$r(\delta)=\bigl(R(\theta_1,\delta),\ldots,R(\theta_k,\delta)\bigr)
\in[0,\infty]^k$.  The \emph{risk set} is
\[
  \cR \;=\; \bigl\{r(\delta):\delta\in\cD\bigr\}
  \;\subset\; [0,\infty]^k.
\]
The risk set is the image of the decision space under the risk map;
its geometry encodes which rules dominate which.  The finite-$\Theta$
representation is the primary geometric display; compact-$\Theta$
extensions are handled through the standard complete-class
regularity assumptions of Section~\ref{sec:geometry}
(the \emph{Standing regularity} hypotheses below).
\end{definition}

\section{Geometry of No-Shame}\label{sec:geometry}

Admissibility has a geometric characterization: a rule is admissible
if and only if its risk vector lies on the lower boundary of the
risk set.  This section establishes convexity, existence of Bayes
rules, and closedness of the upper hull.  For finite $\Theta$ and
finite risk vectors, a supporting hyperplane associates each
admissible risk vector with a prior; under the broader standing
assumptions, the complete-class theorem yields a Bayes rule or a
pointwise limit of Bayes rules.

\paragraph{Standing regularity.}\label{rem:regularity}
Throughout Section~\ref{sec:geometry} we use the standard
compactness and lower-semicontinuity hypotheses required for the
Wald--Blackwell complete-class argument
\citep{wald1950,blackwell1954,berge1963}.  In general Polish sample
spaces with compact action sets, we impose as standing assumptions
compactness of the randomized-rule class $\cD$ in the topology of
pointwise weak convergence and lower semicontinuity of
$\delta\mapsto\int_\Theta R(\theta,\delta)\,\dx\Pi(\theta)$ for
finitely supported priors $\Pi$; compactness of the action space
alone does not imply compactness of the full measurable-rule class.
In the finite-data-space settings used by the explicit witnesses,
these conditions follow from finite-product compactness of
$\Delta(\cA)$ and direct lower-semicontinuity of the risk map
\citep[Ch.~15]{aliprantis2006}.

\subsection{Convexity of the risk set}

\begin{lemma}[Convexity]\label{lem:convexity}
Under Definitions~\ref{def:sdp}--\ref{def:riskset}, $\cR$ is convex.
\end{lemma}

\begin{proof}
Let $\delta_1,\delta_2\in\cD$ and $\lambda\in[0,1]$.  The mixture
$\delta_\lambda$ satisfies
$R(\theta_j,\delta_\lambda)=\lambda R(\theta_j,\delta_1)+(1-\lambda)
R(\theta_j,\delta_2)$ for each $j$, so
$r(\delta_\lambda)=\lambda r(\delta_1)+(1-\lambda)r(\delta_2)\in\cR$.
\end{proof}

\subsection{Existence of Bayes rules}

\begin{lemma}[Existence of Bayes rules]\label{lem:berge}
Under the standing regularity assumptions of
Section~\ref{sec:geometry}, every prior $\Pi$ with finite
integrated risk admits a Bayes rule
$\delta_\Pi\in\arg\min_{\delta\in\cD}
\int_\Theta R(\theta,\delta)\,\dx\Pi(\theta)$.
\end{lemma}

\begin{proof}
By the standing regularity assumption, $\cD$ is compact in the
topology under consideration and the integrated-risk functional
$\delta\mapsto\int_\Theta R(\theta,\delta)\,\dx\Pi(\theta)$ is lower
semicontinuous.  A lower semicontinuous functional attains its
minimum on a compact set.
\end{proof}

\subsection{Closedness of the upper-comprehensive hull}

\begin{proposition}[Upper-comprehensive closedness]\label{prop:closed}
Assume $|\Theta|=k$, the standing regularity assumptions of
Section~\ref{sec:geometry}, and that all attainable risk vectors are
finite, so $\cR\subset[0,\infty)^k$.  Define the
upper-comprehensive hull
\[
  \cR_+ \;:=\; \cR + \Rbb^k_+
  \;=\;
  \{\, s\in[0,\infty)^k : s\ge r(\delta)\text{ coordinatewise for some }\delta\in\cD\,\}.
\]
Then $\cR_+$ is closed in $\Rbb^k$.  More
explicitly, if $s_\alpha\in\cR_+$ and $s_\alpha\to s^*$ along a
net in $\Rbb^k$, then there exists $\delta^*\in\cD$ with
$r(\delta^*)\le s^*$ coordinatewise.
\end{proposition}

\begin{proof}
For each $\alpha$, choose $\delta_\alpha\in\cD$ with
$r(\delta_\alpha)\le s_\alpha$ coordinatewise.  Compactness of
$\cD$ (by the standing regularity assumptions) supplies a subnet
with $\delta_\alpha\to\delta^*$ weakly.  For every coordinate $j$,
lower semicontinuity and $r_j(\delta_\alpha)\le s_{\alpha,j}$ give
\[
 R(\theta_j,\delta^*)
 \le \liminf_\alpha R(\theta_j,\delta_\alpha)
 \le \liminf_\alpha s_{\alpha,j}
 = s_j^*.
\]
Hence $r(\delta^*)\le s^*$ coordinatewise, so
$s^*\in r(\delta^*)+\Rbb^k_+\subseteq\cR_+$, because both vectors
are finite.
\end{proof}

\begin{remark}[Why $\cR$ need not itself be closed]\label{rem:not-closed}
Lower semicontinuity of risk delivers a limit point
$r(\delta^*)\le r^*$, not $r(\delta^*)=r^*$.  The hull $\cR_+$ is
the right closed object for the supporting-hyperplane argument
(Theorem~\ref{thm:supporting}): finite lower-boundary points of
$\cR$ are boundary points of $\cR_+$, but the supporting normal is obtained
from $\cR_+$, where translation by the nonnegative orthant is
allowed.  Continuity of the risk map (rather than mere lower
semicontinuity) would suffice for closedness of $\cR$ itself but is
not assumed here.
\end{remark}

\begin{remark}\label{rem:extended-separation}
Theorem~\ref{thm:supporting} is restricted to finite-valued risk
vectors.  Extended-real risks are retained elsewhere---for example,
in the log-loss plug-in counterexample---but no supporting-hyperplane
conclusion is asserted here for risk vectors with infinite
coordinates.
\end{remark}

\subsection{Lower boundary and admissibility}

\begin{definition}[Lower boundary]\label{def:boundary}
The \emph{lower boundary} of $\cR$ is
\[
  \partial_-\cR \;=\;
  \bigl\{r\in\cR : \nexists\, r'\in\cR,\;
    r'\le r \;\text{coordinatewise with}\; r'\ne r\bigr\}.
\]
\end{definition}

\begin{proposition}[Boundary characterization]\label{prop:boundary}
A rule $\delta$ is Blackwell admissible if and only if
$r(\delta)\in\partial_-\cR$.
\end{proposition}

\begin{proof}
If $r(\delta)\notin\partial_-\cR$, there exists
$r'\in\cR$ with $r'\le r(\delta)$ coordinatewise, $r'\ne r(\delta)$;
the corresponding rule dominates $\delta$.  Conversely, if
$r(\delta)\in\partial_-\cR$, no such $r'$ exists in $\cR$.
\end{proof}

Figure~\ref{fig:riskset} displays the geometry for $|\Theta|=2$;
Figure~\ref{fig:riskset-bernoulli} gives a schematic Bernoulli
illustration in which the plug-in MLE sits off the boundary and
two generic Bayes rules occupy distinct supported points.

\begin{figure}[ht]
\centering
\begin{tikzpicture}[>=Stealth, scale=1.0]
  \draw[->, thick] (-0.3,0) -- (5.8,0)
    node[below] {$R(\theta_1,\delta)$};
  \draw[->, thick] (0,-0.3) -- (0,5.2)
    node[left] {$R(\theta_2,\delta)$};

  \fill[blue!8]
    (0.7,4.6) .. controls (0.9,3.0) and (1.2,2.0) .. (1.6,1.4)
    .. controls (2.2,0.9) and (3.0,0.7) .. (4.8,0.6)
    -- (5.4,1.2) -- (5.4,4.8) -- cycle;
  \draw[blue!50, thick]
    (0.7,4.6) .. controls (0.9,3.0) and (1.2,2.0) .. (1.6,1.4)
    .. controls (2.2,0.9) and (3.0,0.7) .. (4.8,0.6);
  \draw[blue!25, dashed]
    (4.8,0.6) -- (5.4,1.2) -- (5.4,4.8) -- (0.7,4.6);

  \draw[red!70!black, very thick]
    (0.7,4.6) .. controls (0.9,3.0) and (1.2,2.0) .. (1.6,1.4)
    .. controls (2.2,0.9) and (3.0,0.7) .. (4.8,0.6);
  \node[red!70!black, right] at (4.9,0.45) {\small $\partial_-\cR$};

  \fill[black] (1.6,1.4) circle (2.5pt);
  \node[anchor=south east, font=\small] at (1.5,1.5)
    {$r^{*}=r(\delta_{\Pi})$};

  \draw[black!55, thick, dashed]
    (0.0,2.6) -- (3.5,0.0);
  \node[black!55, anchor=west, font=\footnotesize]
    at (0.9,0.5)
    {$\langle\pi,r\rangle=\text{const}$};

  \draw[->, thick, black!85] (1.6,1.4) -- (2.35,2.4);
  \node[anchor=south, font=\footnotesize] at (2.35,2.5)
    {$\pi\propto\Pi$};

  \fill[gray!70] (3.0,3.2) circle (2pt);
  \node[gray!70!black, anchor=west, font=\footnotesize]
    at (3.2,3.25) {dominated};
  \draw[->, gray!50, thin] (3.0,3.2) -- (2.32,1.08);

  \node[blue!50!black, font=\small] at (4.2,3.6) {$\cR$};
\end{tikzpicture}
\caption{Risk set geometry for $|\Theta|=2$.  The convex risk set
  $\cR$ (shaded) maps each decision rule to a risk vector.  The lower
  boundary $\partial_-\cR$ (bold curve) contains all admissible
  rules.  At an admissible point $r^*$, the supporting hyperplane
  (dashed line), with normal $\pi$, identifies the prior $\Pi$
  defining the Bayes problem that $r^*$ solves
  (Theorem~\ref{thm:supporting}).  Interior points are dominated.}
\label{fig:riskset}
\end{figure}

\begin{figure}[ht]
\centering
\resizebox{0.62\linewidth}{!}{%
\begin{tikzpicture}[>=Stealth, scale=1.0,
  ptdot/.style={circle, fill, inner sep=1.6pt},
  hp/.style={thick, dashed}]
  \draw[->, thick] (-0.15,0) -- (5.6,0)
    node[below=2pt, font=\small] {$R(0.3,\delta)$};
  \draw[->, thick] (0,-0.15) -- (0,5.6)
    node[left=2pt, font=\small] {$R(0.7,\delta)$};

  \fill[blue!8]
    plot[domain=0.9:5.4, samples=80] (\x, {4/\x})
    -- (5.4,5.4) -- (0.9,5.4) -- cycle;

  \draw[red!70!black, very thick]
    plot[domain=0.9:5.4, samples=80] (\x, {4/\x});
  \node[red!70!black, anchor=west, font=\small] at (5.0,0.95)
    {$\partial_-\cR$};
  \node[blue!50!black, font=\small] at (4.2,3.4) {$\cR$};

  \node[ptdot, fill=black] at (2.0,2.0) {};
  \node[anchor=north east, font=\footnotesize] at (1.95,1.95)
    {$\delta_{\Pi_1}$};
  \draw[hp, black!60] (0.4,3.6) -- (3.6,0.4);
  \node[black!60, font=\scriptsize, anchor=west]
    at (3.68,0.40) {$\Pi_1$};

  \node[ptdot, fill=blue!60!black] at (1.4,2.857) {};
  \node[anchor=south east, font=\footnotesize, blue!60!black]
    at (1.35,2.91) {$\delta_{\Pi_2}$};
  \draw[hp, blue!50] (0.35,5.0) -- (2.45,0.7);
  \node[blue!50, font=\scriptsize, anchor=north]
    at (2.45,0.58) {$\Pi_2$};

  \node[ptdot, fill=gray!70] at (3.4,3.4) {};
  \node[gray!75!black, anchor=south west, font=\footnotesize]
    at (3.5,3.45) {$\hat p_n^{\mathrm{pi}}$};
  \draw[->, thick, gray!60] (3.3,3.3) -- (2.18,2.18);
  \node[gray!70!black, font=\scriptsize, anchor=south west,
        rotate=45]
    at (2.55,2.55) {dominated};

  \node[gray!65!black, font=\scriptsize, anchor=west]
    at (3.5,5.05)
    {$L(\theta,\hat p_n^{\mathrm{pi}})=+\infty$};
  \node[gray!65!black, font=\scriptsize, anchor=west]
    at (3.5,4.65)
    {on $\{S_n\in\{0,n\}\}$};
\end{tikzpicture}%
}
\caption{Schematic risk set for Bernoulli log-loss prediction with
  $\Theta=\{0.3,\,0.7\}$ and $n=10$.  The curve and plotted
  coordinates are illustrative rather than numerically computed
  risk values.  The boundary points $\delta_{\Pi_1}$ and
  $\delta_{\Pi_2}$ represent Bayes rules under two full-support
  priors on this two-point parameter space, depicted by tangent
  (dashed) supporting hyperplanes.  The plug-in MLE
  $\hat p_n^{\mathrm{pi}}=S_n/n$ is shown by a finite schematic
  marker, but its actual extended-real risk vector is
  $(+\infty,+\infty)$: it is dominated because it assigns probability
  zero to events that occur with positive probability, producing
  infinite log-loss contributions on
  $\{S_n\in\{0,n\}\}$.  Its extended-real risk coordinates are
  $+\infty$ at every finite $n$; the dot marks a finite proxy
  position for visual comparison.}
\label{fig:riskset-bernoulli}
\end{figure}

\subsection{Supporting hyperplanes and Bayes rules}

\paragraph{Notation convention ($\pi$ vs.\ $\Pi$).}
Throughout the paper we maintain a strict notational separation
between two related but distinct objects.  The Greek letter $\pi$
(lower case) denotes the \emph{geometric supporting-hyperplane
normal vector} at a boundary point of $\cR_+$; it lives in
$\Rbb^k_+\setminus\{0\}$.  The capital $\Pi$ denotes the
\emph{normalized prior probability distribution} on $\Theta$
obtained by setting $\Pi=\pi/\|\pi\|_1$.  Thus $\pi$ carries
geometric content (it is the dual variable of a convex separation)
while $\Pi$ carries probabilistic content (it integrates risk into
Bayes risk).  The two are equivalent up to normalization but are
not interchangeable in the prose.

\paragraph{Intuition for Theorem~\ref{thm:supporting}.}
Under the finite-$\Theta$, finite-risk assumptions below, at any
point on the lower boundary of the risk set there is a
supporting hyperplane that touches the boundary at that point
and lies weakly below every risk vector achievable by some
decision rule.  Convexity of $\cR$ delivers the hyperplane;
upper-comprehensiveness of $\cR_+$ forces its normal to have
nonnegative coordinates; nonnegativity lets us interpret the
normal, after normalization, as a prior over $\Theta$.  The
theorem says that for every admissible risk vector there is such a
prior, and that the admissible rule is then a Bayes rule under it.

\begin{theorem}[Supporting hyperplane identification]%
\label{thm:supporting}
Assume $|\Theta|=k$, the standing regularity conditions of
Section~\ref{sec:geometry}, and finite-valued risk vectors
$\cR\subset\Rbb^k$.  Let $\delta^*\in\cD$ have
$r^*=r(\delta^*)\in\partial_-\cR$.
Then there exists $\pi\in\Rbb^k_+\setminus\{0\}$ such that
\[
  \sum_{j=1}^k \pi_j R(\theta_j,\delta^*)
  \;\le\;
  \sum_{j=1}^k \pi_j R(\theta_j,\delta)
  \qquad \text{for all }\delta\in\cD.
\]
Setting $\Pi=\pi/\|\pi\|_1$ defines a prior on $\Theta$, and
$\delta^*$ is a Bayes rule:
\[
  \delta^* \;\in\; \arg\min_{\delta\in\cD}
    \int R(\theta,\delta)\,\dx\Pi(\theta).
\]
\end{theorem}

\begin{proof}
Consider the upper-comprehensive hull
$\cR_+ := \cR + \Rbb^k_+$.  By
Proposition~\ref{prop:closed}, $\cR_+$ is closed; by
Lemma~\ref{lem:convexity} it is convex.  Thus $\cR_+$ is a nonempty
closed convex subset of the finite-dimensional space $\Rbb^k$.
Moreover, $r^*$ lies on its boundary.  Indeed, if $r^*$ were an
interior point, then $r^*-\epsilon\mathbf 1\in\cR_+$ for some
$\epsilon>0$.  There would then be an $r\in\cR$ with
$r\le r^*-\epsilon\mathbf 1<r^*$, contradicting
$r^*\in\partial_-\cR$.
By the supporting-hyperplane theorem for closed convex sets
\citep[Thm.~11.6]{rockafellar1970}, there exists
$\pi\in(\Rbb^k)^*\setminus\{0\}$ with
\begin{equation}
  \langle\pi,r\rangle \;\ge\; \langle\pi,r^*\rangle
  \qquad \text{for all }r\in\cR_+.
  \label{eq:supp-hyper}
\end{equation}
We show $\pi\ge 0$ coordinatewise.  Fix $j\in\{1,\dots,k\}$ and
$t>0$.  By upper-comprehensiveness of $\cR_+$,
$r^*+t e_j \in \cR_+$ where $e_j$ is the $j$-th standard basis
vector.  Applying \eqref{eq:supp-hyper} at $r=r^*+t e_j$ yields
$\langle\pi,r^*\rangle + t\pi_j \ge \langle\pi,r^*\rangle$, i.e.\
$t\pi_j\ge 0$.  Since $t>0$ was arbitrary, $\pi_j\ge 0$.  As $j$ was
arbitrary, $\pi\in\Rbb^k_+\setminus\{0\}$.
Because $\cR\subseteq\cR_+$, \eqref{eq:supp-hyper} restricts to
$\sum_j \pi_j R(\theta_j,\delta) \ge \sum_j \pi_j R(\theta_j,\delta^*)$
for every $\delta\in\cD$.  Normalizing $\Pi=\pi/\|\pi\|_1$
identifies $\delta^*$ with a Bayes minimizer under~$\Pi$.
\end{proof}

\begin{remark}[Scope of Theorem~\ref{thm:supporting}]\label{rem:finite-theta}
The finite-$\Theta$, finite-risk statement is the geometric display
used in this paper.  Extended-real problems require the separation
argument to be restricted to a finite-risk closed convex portion;
the Bayes-implies-admissibility results below do not require every
competitor to have finite risk.  Compact-$\Theta$ versions require the standard
complete-class regularity assumptions; Corollary~\ref{cor:noshame-compact}
gives the support/continuity condition needed for the
Bayes-implies-no-shame direction.  See \citet{blackwell1954} and
\citet{wald1950} for the classical general development.
\end{remark}

The converse direction (Bayes $\Rightarrow$ no-shame) requires
support assumptions: a Bayes rule under a prior may be improved at
a point of prior mass zero without changing the integrated risk
under that prior, and lower semicontinuity alone does not exclude
this.  Corollaries~\ref{cor:noshame} and~\ref{cor:noshame-compact}
state the finite-$\Theta$ (full-support prior) and compact-$\Theta$
(full topological support, continuous risk) versions
respectively; the general statement is the complete class theorem
below.

\begin{theorem}[Wald--Blackwell complete class, under standard
  regularity {\citep{wald1950,blackwell1954}}]\label{thm:complete-class}
Under Definition~\ref{def:sdp} together with the standing
compactness, lower-semicontinuity, and closed-risk-set conditions
of Section~\ref{sec:geometry}, every
Blackwell admissible rule is a Bayes rule with respect to some prior
$\Pi$ on $\Theta$, or a pointwise limit of Bayes rules.
Equivalently, the class of Bayes rules is essentially complete.
\end{theorem}

\subsection{No-shame strategies}

\begin{definition}[No-shame strategy]\label{def:noshame}
A rule $\delta\in\cD$ is a \emph{no-shame strategy} if it is
Blackwell admissible; equivalently (by
Proposition~\ref{prop:boundary}), if $r(\delta)\in\partial_-\cR$.
\end{definition}

\begin{corollary}[Full-support Bayes implies no-shame (finite~$\Theta$)]%
\label{cor:noshame}
Let $\Theta=\{\theta_1,\dots,\theta_k\}$ be finite and let $\Pi$ be a
prior with $\Pi(\theta_j)>0$ for every $j$.  Every proper Bayes
rule $\delta_\Pi$ with finite integrated risk under such a
\emph{full-support} prior is Blackwell admissible (no-shame).
\end{corollary}

\begin{proof}
Suppose, for contradiction, that $\delta'\in\cD$ dominates
$\delta_\Pi$: $R(\theta_j,\delta')\le R(\theta_j,\delta_\Pi)$ for
all $j$, with strict inequality at some $j_0$.  Multiplying by
$\Pi(\theta_j)>0$ and summing,
$\sum_j \Pi(\theta_j) R(\theta_j,\delta') < \sum_j \Pi(\theta_j) R(\theta_j,\delta_\Pi)$,
contradicting Bayes optimality of $\delta_\Pi$ under $\Pi$.
\end{proof}

\begin{remark}[Direction of the equivalence]\label{rem:noshame-direction}
The converse of Corollary~\ref{cor:noshame} requires care.
In the finite-risk setting, Theorem~\ref{thm:supporting} produces a
supporting normal $\pi$ at every $r^*\in\partial_-\cR$, but the normalized prior
$\Pi=\pi/\|\pi\|_1$ may assign zero mass to some
coordinates.  Hence every Blackwell admissible rule is Bayes
\emph{with respect to some prior}, possibly without full
support, or is a pointwise limit of Bayes rules; this is the
Wald--Blackwell complete class theorem
(Theorem~\ref{thm:complete-class}).  We are therefore careful in
the rest of the paper to distinguish the implication
\textit{full-support Bayes $\Rightarrow$ no-shame} from its
converse.  Admissibility under a non-full-support prior is
witnessed directly by the risk-vector geometry (as in the
point-mass example used in
Theorem~\ref{thm:extended-separation}\,(i)).
\end{remark}

\begin{corollary}[Full-support Bayes implies no-shame (compact-$\Theta$ version)]%
\label{cor:noshame-compact}
Let $\Theta$ be compact metrizable, $\Pi$ a Borel prior with full
topological support on $\Theta$, and assume the risk function
$\theta\mapsto R(\theta,\delta)$ is continuous for every
$\delta\in\cD$.  Then every proper Bayes rule $\delta_\Pi$ with
finite integrated risk is Blackwell admissible.
\end{corollary}

\begin{proof}
Suppose, for contradiction, that $\delta'$ dominates $\delta_\Pi$.
Then $R(\theta_0,\delta')<R(\theta_0,\delta_\Pi)$ at some
$\theta_0\in\Theta$, and $R(\theta,\delta')\le R(\theta,\delta_\Pi)$
elsewhere.  By continuity of $\theta\mapsto R(\theta,\delta')-
R(\theta,\delta_\Pi)$, the strict inequality extends to an open
neighborhood $U\ni\theta_0$.  Full topological support of $\Pi$
gives $\Pi(U)>0$, hence
\[
  \int_\Theta \bigl[R(\theta,\delta_\Pi)-R(\theta,\delta')\bigr]\dx\Pi(\theta)
  \;\ge\;
  \int_U \bigl[R(\theta,\delta_\Pi)-R(\theta,\delta')\bigr]\dx\Pi(\theta)
  \;>\;0,
\]
contradicting Bayes optimality of $\delta_\Pi$.
\end{proof}

\begin{remark}[Why continuity and full support are needed]\label{rem:noshame-split}
Positive prior mass at every point is appropriate only in the finite
case.  For compact $\Theta$ with general priors, a strict improvement at a
single point $\theta_0$ with $\Pi(\{\theta_0\})=0$ does not
automatically reduce the integrated Bayes risk.  The
compact-$\Theta$ statement adds the standard continuity-of-risk and
full-support hypotheses, which suffice for the dominance to extend
to a neighborhood of positive $\Pi$-measure.  For log loss on the
Bernoulli model with $\Theta=(0,1)$, continuity holds whenever the
predictor avoids boundary predictions (cf.\
Theorem~\ref{thm:insuff}); the Bayes posterior predictives
considered in this paper satisfy this requirement.
\end{remark}

\begin{remark}[Order-theoretic structure]\label{rem:order}
Admissibility is defined by the coordinatewise partial order on
$\Rbb^k$, not by any metric.  The supporting hyperplane in
Theorem~\ref{thm:supporting} is the linear-algebraic instrument
for locating the prior $\Pi$ that rationalizes $\delta^*$; the
dominance relation itself depends only on the order structure of
$\Rbb^k$.  When validity constraints are imposed (anytime-valid
error control or marginal coverage), the optimization in
Theorem~\ref{thm:supporting} restricts to a feasible subset of
$\cD$; the constrained Bayes formulation in
Section~\ref{sec:constrained} makes this precise.
\end{remark}

\begin{remark}[Structural scope of the admissibility results]%
\label{rem:structural-scope}
The existence and complete-class conclusions in this section rely
on the stated compactness, lower-semicontinuity, and closed-risk-set
assumptions of the classical Berge--Wald--Blackwell framework
\citep{berge1963,wald1950,blackwell1954}.  Proposition~\ref{prop:closed}
and Theorem~\ref{thm:supporting} are explicitly finite-risk results.
Extended-real log loss remains allowed in the dominance examples,
including the plug-in counterexample, but those arguments do not
invoke finite-dimensional supporting-hyperplane separation.
\end{remark}

\section{Martingale Layer}\label{sec:martingale}

Bayesian posterior predictive means add a dynamic notion to the
risk-set geometry: under a single prior they are martingales under
the prior predictive law \citep{doob1949,fong2024}.  This property
is central to anytime-valid inference but is neither necessary nor
sufficient for fixed-sample Blackwell admissibility in general.
The next definition separates the three coherence notions used
below.

\subsection{Three notions of coherence}\label{sec:three-coherence}

We distinguish three coherence properties to locate the plug-in
counterexample.

\begin{definition}[Three coherence notions]\label{def:three-coherence}
Let $(\delta_n)_{n\ge 1}$ be a sequence of decision rules on
$(\cX^n)_{n\ge 1}$.  For (C2) and (C3), the action is a predictive
distribution (or its scalar predictive mean) in a conditionally
i.i.d.\ model.
\begin{enumerate}[label=\textup{(C\arabic*)},leftmargin=*]
\item \emph{Fixed-sample Blackwell admissibility.}  Each $\delta_n$
  is Blackwell admissible in the sense of
  Definition~\ref{def:admissibility}: no $\delta'\in\cD$ satisfies
  $R(\theta,\delta')\le R(\theta,\delta_n)$ for all $\theta$ with
  strict inequality somewhere.
\item \emph{Sequential Bayes coherence under $\Pi$.}  There exists
  a single prior $\Pi$ such that, for every $n$, $\delta_n$ is its
  posterior predictive Bayes act; the associated posterior
  predictive mean sequence $(m_n)$ of Definition~\ref{def:ppseq} is
  then a martingale under the prior predictive law
  $\tilde P=\int P_\theta\,\dx\Pi(\theta)$.
\item \emph{Self-consistency under $\hat P$.}  The sequence
  $(\delta_n)$ is a martingale under its own predictive law $\hat
  P$, in which $X_{n+1}\mid\cF_n$ is distributed as $\delta_n$
  itself.
\end{enumerate}
\end{definition}

The three notions are logically distinct.  (C2) implies (C1) at
each $n$ under the assumptions of
Corollary~\ref{cor:noshame} or~\ref{cor:noshame-compact}.  The
implication (C3)$\Rightarrow$(C1) fails: the plug-in MLE
$\hat p_n^{\mathrm{pi}}=S_n/n$ satisfies (C3) but is strictly
dominated under log loss for every $\theta\in(0,1)$ and every
$n\ge 1$ (Theorem~\ref{thm:insuff}).  Likewise (C1) does not imply
(C2): a Blackwell-admissible rule may be Bayes under a different
prior at each $n$, in which case the sequence is not coherent under
any single $\Pi$.  Proposition~\ref{prop:bayes-martingale} below is
the precise statement that (C2) holds: the martingale claim is
under the \emph{prior} predictive law $\tilde P$, not under any
data-dependent predictive law $\hat P$.

Notion (C3) is the self-audit property at the heart of the
well-calibrated Bayesian of \citet{dawid1982}: a forecaster
assessed under its own predictive law finds nothing to correct.
The plug-in pathology of Theorem~\ref{thm:insuff} is exactly the
failure of (C3)$\Rightarrow$(C1): self-consistency under the
predictor's own law is not a substitute for non-dominance under the
true data-generating process $P_\theta$.  The Bayes corrections
that resolve the pathology
(e.g.\ $\mathrm{Beta}(\tfrac12,\tfrac12)$ shrinkage) restore (C2),
which in turn delivers (C1) by Corollary~\ref{cor:noshame}.

\begin{definition}[Posterior predictive mean sequence]\label{def:ppseq}
Let $\Pi$ be a prior on $\Theta$ with $\cF_n=\sigma(X_1,\ldots,X_n)$.
In a conditionally i.i.d.\ model, for every bounded measurable test
function $h$ define
$\mu_h(\theta)=\Ebb_\theta[h(X_1)]$.  The process
$m_n(h)=\Ebb_\Pi[\mu_h(\theta)\mid\cF_n]$ is a scalar martingale
under the prior predictive law
$\tilde P=\int P_\theta\,\dx\Pi(\theta)$.  It is the posterior mean
of the next-observation functional.  In the Bernoulli model, taking
$h(x)=x$ recovers the \emph{posterior predictive mean sequence}
\[
  m_n \;=\; \Ebb_\Pi[\theta \mid \cF_n], \qquad n\ge 0,
\]
with $m_0=\Ebb_\Pi[\theta]$, which is the scalar version used
throughout the paper.
\end{definition}

\begin{proposition}[Bayes implies martingale]\label{prop:bayes-martingale}
Under Definition~\ref{def:ppseq}, $(m_n)_{n\ge 0}$ is a martingale
with respect to $(\cF_n)$ under the prior predictive measure
$\tilde P=\int P_\theta\,\dx\Pi(\theta)$.
\end{proposition}

\begin{proof}
By the tower property under $\tilde P$:
$\Ebb_{\tilde P}[m_n\mid\cF_{n-1}]
=\Ebb_{\tilde P}[\Ebb_\Pi[\theta\mid\cF_n]\mid\cF_{n-1}]
=\Ebb_\Pi[\theta\mid\cF_{n-1}]
=m_{n-1}$ a.s.
Integrability holds since $\theta$ is bounded on compact $\Theta$.
\end{proof}

\subsection{Self-consistency is not sufficient for admissibility}

\begin{theorem}[Self-consistency does not imply Blackwell admissibility]%
\label{thm:insuff}
In the Bernoulli model
$X_i\overset{\mathrm{iid}}{\sim}\mathrm{Bern}(\theta)$,
$\theta\in(0,1)$, under log loss
$L(\theta,p)=-\theta\log p-(1-\theta)\log(1-p)$:
\begin{enumerate}[label=(\roman*),nosep]
\item Every Bayesian posterior predictive sequence $(m_n)$ is a
  martingale under the prior predictive measure
  (Proposition~\ref{prop:bayes-martingale}).
\item The plug-in rule $\hat p_n^{\mathrm{pi}}=S_n/n$ satisfies
  the martingale condition under its own predictive measure
  $\hat P$ (where $X_t\mid X_{1:t-1}\sim\mathrm{Bern}(\hat p_{t-1}^{\mathrm{pi}})$)
  for $n\ge 2$, with any fixed initialization
  $\hat p_0\in(0,1)$.
\item $\hat p_n^{\mathrm{pi}}$ is strictly dominated by the Bayes
  rule $\hat p_n^B=(S_n+\tfrac12)/(n+1)$ under
  $\mathrm{Beta}(\tfrac12,\tfrac12)$ prior, for every $n\ge 1$ and
  $\theta\in(0,1)$.
\end{enumerate}
Hence $\hat p_n^{\mathrm{pi}}$ is Blackwell-inadmissible; on every
finite submodel its risk vector lies off $\partial_-\cR$.
Self-consistency of a predictor under its own predictive law is not
sufficient for Blackwell admissibility.
\end{theorem}

\begin{proof}
\emph{(i)} Proposition~\ref{prop:bayes-martingale}.

\emph{(ii)} Under $\hat P$,
$\Ebb_{\hat P}[\hat p_n^{\mathrm{pi}}\mid X_{1:n-1}]
=(S_{n-1}+\hat p_{n-1}^{\mathrm{pi}})/n
=\bigl(S_{n-1}+S_{n-1}/(n-1)\bigr)/n
=S_{n-1}/(n-1)
=\hat p_{n-1}^{\mathrm{pi}}$ a.s.

\emph{(iii)} Under log loss, the risk decomposes as
\[
  R(\theta,\hat p_n)
  = \Ebb_\theta\bigl[\KL(\mathrm{Bern}(\theta)\|\mathrm{Bern}(\hat p_n))\bigr]
    + H(\theta),
\]
where $H(\theta)=-\theta\log\theta-(1-\theta)\log(1-\theta)$ is the
binary entropy.  Hence the excess risk is
\[
  R(\theta,\hat p_n^{\mathrm{pi}})-R(\theta,\hat p_n^B)
  = \Ebb_\theta\bigl[\KL(\mathrm{Bern}(\theta)\|\mathrm{Bern}(\hat p_n^{\mathrm{pi}}))
    -\KL(\mathrm{Bern}(\theta)\|\mathrm{Bern}(\hat p_n^B))\bigr].
\]
Since $\hat p_n^{\mathrm{pi}}=S_n/n\in\{0,\tfrac{1}{n},\ldots,1\}$
and $P_\theta(S_n=0)=(1-\theta)^n>0$ for every $\theta\in(0,1)$,
the predictor $\hat p_n^{\mathrm{pi}}=0$ assigns probability zero to
$X_{n+1}=1$, an event with probability $\theta>0$; thus
$\KL(\mathrm{Bern}(\theta)\|\mathrm{Bern}(\hat p_n^{\mathrm{pi}}))=+\infty$ on the
event $\{S_n=0\}$ and $R(\theta,\hat p_n^{\mathrm{pi}})=+\infty$.
Meanwhile $\hat p_n^B=(S_n+\tfrac12)/(n+1)\in(0,1)$ for all
$S_n\in\{0,\ldots,n\}$, so $R(\theta,\hat p_n^B)<\infty$.  Hence
$\hat p_n^B$ strictly dominates $\hat p_n^{\mathrm{pi}}$ for all
$\theta\in(0,1)$ and $n\ge 1$, giving
$\hat p_n^{\mathrm{pi}}$ is Blackwell-inadmissible by direct
dominance.  On any finite submodel, Proposition~\ref{prop:boundary}
also places its risk vector off $\partial_-\cR$.  By
Definition~\ref{def:noshame}, $\hat p_n^{\mathrm{pi}}$ is not
no-shame.
\end{proof}

The failure is not that the plug-in lacks a martingale law; it is
that the martingale law is the wrong one for Blackwell risk.  The
prior predictive law $\tilde P$, the plug-in self-predictive law
$\hat P$, and the true model law $P_\theta$ are three distinct
measures on the same path space; only the first two make the
respective Bayes / plug-in sequences martingales, and only the
third governs admissibility.

\begin{remark}[Role of extended-real risk]\label{rem:extended-real}
The dominance in part~(iii) requires $+\infty$ risk
(Definitions~\ref{def:sdp}~and~\ref{def:risk}).  The boundary
pathology ($\KL(\mathrm{Bern}(\theta)\|\mathrm{Bern}(\hat p_n^{\mathrm{pi}}))
=+\infty$ on $\{S_n=0\}$) is eliminated if loss is bounded, but
bounded losses would exclude the unbounded log loss used here.
Extended-real risk is therefore essential, not a technicality.
\end{remark}

\section{Criterion Separation}\label{sec:separation}

Sections~\ref{sec:geometry}--\ref{sec:martingale} established
Blackwell admissibility as the first geometry.  We now introduce two
additional admissibility criteria, anytime-valid sequential
inference and marginal coverage validity, each operating on a
different space of procedures with a different partial order, and
prove that the three resulting classes are pairwise non-nested.

Because predictors, $e$-processes, prediction sets, and online
strategies have different types, we place them in a common product
space $\Sigma$ and attach each criterion to its own coordinate.

\subsection{The predictive-systems meta-space \texorpdfstring{$\Sigma$}{Sigma}}
\label{sec:meta-space}

\begin{definition}[Predictive system]\label{def:predictive-system}
Let $\mathcal Z$ be a Polish single-observation space and let
$\Omega=\mathcal Z^\infty$ be the canonical path space, equipped with the
product $\sigma$-algebra $\cF_\infty$ and the natural filtration
$\cF_n=\sigma(Z_1,\ldots,Z_n)$.  In supervised prediction take
$Z_i=(X_i,Y_i)$; when there is no separate response, identify
$Z_i=X_i$.  Let $\cP=\{P_\theta:\theta\in\Theta\}$
be a statistical model on $(\Omega,\cF_\infty)$, with
$\Theta\subset\Rbb^d$ compact and metrizable.  A \emph{predictive
system} on this filtered space is a quadruple
\[
  \Delta = (\delta, E, \hat C, \sigma)
\]
whose four components are $(\cF_n)$-adapted processes:
\begin{enumerate}[label=\textup{(\alph*)},nosep]
\item $\delta=(\delta_n)_{n\ge 1}$, where each $\delta_n$ is
  $\cF_n$-measurable, taking values in $\Delta(\cA)$ with $\cA$
  compact metric and $\Delta(\cA)$ equipped with the weak topology
  (\emph{predictor component});
\item $E=(E_n)_{n\ge 0}$, where $E_0\equiv 1$, $E_n\ge 0$, and each
  $E_n$ is $\cF_n$-measurable (\emph{test component});
\item $\hat C=(\hat C_n)_{n\ge 0}$, where each $\hat C_n$ is an
  $\cF_n$-measurable prediction-set process (\emph{set component});
\item $\sigma=(\sigma_n)_{n\ge 1}$, where each $\sigma_n$ is
  $\cF_{n-1}$-measurable, taking values in $\cA$ (\emph{strategy
  component}).
\end{enumerate}
The collection of all such quadruples is denoted $\Sigma$.  Fix
once and for all a \emph{default inactive system}
$\Delta^\varnothing=(\delta^\varnothing,E^\varnothing,
\hat C^\varnothing,\sigma^\varnothing)$, where
$\delta^\varnothing_n\equiv a_\varnothing$ and
$\sigma^\varnothing_n\equiv a_\varnothing$ for a fixed default
action $a_\varnothing\in\cA$, $E^\varnothing_n\equiv 1$, and
$\hat C^\varnothing_n\equiv\cY$.  A coordinate of $\Delta$ is
\emph{active} if it differs from the corresponding default
coordinate on a set of positive $P_\theta$-probability for some
$\theta\in\Theta$.  In the Bernoulli illustrations below we take
the default inactive action to be $a_\varnothing=\tfrac12$, so the
constant-zero predictor and constant-zero strategy used as failure
witnesses in Theorem~\ref{thm:separation} are active coordinates.
\end{definition}

\paragraph{Measurability.}
Every witness
constructed in Theorems~\ref{thm:separation}
and~\ref{thm:extended-separation}---the Bayes predictors, the
plug-in rule, the likelihood-ratio $e$-process, the conformal
threshold sets, the KT strategy, and the constant comparison
rules---is given by an explicit measurable formula.  All witnesses
are ordinary Borel-adapted processes on the canonical path space.

\begin{definition}[Criterion-relative admissibility on $\Sigma$]%
\label{def:meta-adm}
For a predictive system $\Delta=(\delta,E,\hat C,\sigma)\in\Sigma$:
\begin{itemize}[leftmargin=*,nosep]
\item $\Delta\in\mathfrak{B}$ iff $\delta$ is Blackwell admissible per
  Definition~\ref{def:admissibility};
\item $\Delta\in\mathfrak{A}$ iff $E\in\cC_{\mathrm{AV}}$
  (Definition~\ref{def:av-class}) and $E$ is anytime-valid
  admissible in the sense of Theorem~\ref{thm:ramdas};
\item $\Delta\in\mathfrak{C}$ iff $\hat C$ is coverage-admissible per
  Definition~\ref{def:cov-adm}, within the stated comparison class
  (Remark~\ref{rem:cov-class});
\item $\Delta\in\mathfrak{D}$ iff $\sigma$ is CApp boundary-feasible
  per Definition~\ref{def:capp}.
\end{itemize}
At the component level, each criterion's admissibility (or
feasibility, for $\mathfrak{D}$) is defined on the relevant
procedure class; the four lifted classes
$\mathfrak{B},\mathfrak{A},\mathfrak{C},\mathfrak{D}\subseteq\Sigma$
above are their coordinatewise lifts via
Definition~\ref{def:predictive-system}.
\end{definition}

\begin{remark}[Lifting and active coordinates]\label{rem:lift}
Every classical procedure ($e$-process, conformal set, defensive
forecaster, etc.) lifts to a predictive system by setting the
unused components to their default-inactive values.  The
criterion separation theorem below (Theorem~\ref{thm:separation})
requires \emph{more}: each witness $\Delta\in\Sigma$ must be
active in both the membership and the failure coordinate, so
that the failure under the second criterion is a substantive
dominance failure within an active coordinate rather than
categorical inapplicability.  We verify activeness case by case
in the proof of Theorem~\ref{thm:separation}.
\end{remark}

\subsection{Anytime-valid admissibility}

\begin{definition}[Anytime-valid constraint class]\label{def:av-class}
Let $\cH_0$ be a null class, possibly composite.  The
\emph{anytime-valid class} is
\[
  \cC_{\mathrm{AV}} = \bigl\{
    (E_t)_{t\ge 0} : E_0=1,\; (E_t)\text{ is adapted},\; E_t\ge 0,\;
    \sup_{\Pbb\in\cH_0}\Ebb_\Pbb[E_\tau]\le 1
    \text{ for every stopping time }\tau
  \bigr\}.
\]
Elements of $\cC_{\mathrm{AV}}$ are called \emph{e-processes}.  By
Ville's inequality \citep{ville1939}, every $E\in\cC_{\mathrm{AV}}$
provides anytime-valid type-I error control at level $\alpha$.
\end{definition}

\begin{definition}[Admissibility within $\cC_{\mathrm{AV}}$]%
\label{def:av-adm}
For $E,E'\in\cC_{\mathrm{AV}}$, say $E'$ \emph{dominates} $E$,
written $E'\succeq_{\mathrm{AV}}E$, if $E'_t\ge E_t$ almost surely
for every $t$ under every $\Pbb\in\cH_0$, with strict inequality at
some time with positive probability under at least one
$\Pbb\in\cH_0$.  An element of $\cC_{\mathrm{AV}}$ is
\emph{AV-admissible} if no element of $\cC_{\mathrm{AV}}$ dominates
it in this order.
\end{definition}

\begin{theorem}[Point-null characterization; \citealt{ramdas2022}]\label{thm:ramdas}
If $\cH_0=\{P_0\}$ is a point null, a procedure in
$\cC_{\mathrm{AV}}$ is AV-admissible if and only if it is a
nonnegative $P_0$-martingale.  For composite $\cH_0$, being a
martingale under every $P\in\cH_0$ is sufficient but not necessary
for AV-admissibility.
\end{theorem}

\begin{remark}\label{rem:av-distinct}
Theorem~\ref{thm:ramdas} establishes admissibility relative to
$\cC_{\mathrm{AV}}$ and the criterion of type-I error control at
every stopping time.  This is distinct from Blackwell admissibility,
which requires no domination under a loss $L(\theta,\delta)$ over
all of $\Theta$.  Figure~\ref{fig:av-cone-detail} depicts the
anytime-valid feasible region, an admissible point-null martingale
path, and the
stopped-process geometry behind Ville's inequality.
\end{remark}

\begin{figure}[ht]
\centering
\begin{tikzpicture}[>=Stealth, scale=1.05]
  \draw[->, thick] (-0.4,0) -- (9.0,0)
    node[below=4pt] {sample size $t$};
  \draw[->, thick] (0,-0.4) -- (0,6.4)
    node[left=4pt] {$E_t$};

  \fill[green!5]
    (0.4,1.2) -- (8.2,5.4) -- (8.2,0.05) -- (0.4,0.05) -- cycle;
  \draw[green!50!black, thick, dashed]
    (0.4,1.2) -- (8.2,5.4);
  \draw[green!50!black, thick, dashed]
    (0.4,1.2) -- (8.2,0.05);
  \node[green!40!black, font=\small] at (6.4,0.9)
    {$\cC_{\mathrm{AV}}$: anytime-valid class};

  \draw[black!30, thin, dashed] (0,1.2) -- (8.5,1.2);
  \node[black!50, font=\footnotesize, left] at (0,1.2) {$E_0\!=\!1$};

  \draw[red!50, thin, dashed] (0,4.4) -- (8.5,4.4);
  \node[red!50, font=\footnotesize, left] at (0,4.4) {$1/\alpha$};

  \draw[red!70!black, very thick]
    (0.4,1.2) -- (1.2,1.5) -- (2.0,1.15) -- (2.8,1.7)
    -- (3.6,1.4) -- (4.4,1.8) -- (5.2,1.55) -- (6.0,1.95)
    -- (6.8,1.7) -- (7.6,2.05) -- (8.2,1.8);
  \fill[red!70!black] (8.2,1.8) circle (2.5pt);
  \node[red!70!black, font=\small, above right=2pt] at (8.2,1.8)
    {admissible};

  \draw[gray, thick, densely dotted]
    (0.4,1.2) -- (1.2,1.9) -- (2.0,2.9) -- (2.8,2.4)
    -- (3.6,3.6) -- (4.4,3.2) -- (5.2,4.5) -- (6.0,4.0)
    -- (6.8,5.2) -- (7.6,4.7) -- (8.2,5.6);
  \fill[gray] (8.2,5.6) circle (2.5pt);
  \node[gray, font=\small, right=2pt] at (8.2,5.7)
    {infeasible};

  \draw[->, thick, orange!70!black] (3.6,1.4) -- (3.6,0.2);
  \node[orange!70!black, font=\footnotesize] at (3.6,-0.15) {$\tau$};
  \draw[orange!70!black, very thick]
    (3.6,1.4) -- (4.4,1.4) -- (5.2,1.4) -- (6.0,1.4)
    -- (6.8,1.4) -- (7.6,1.4) -- (8.2,1.4);
  \node[orange!70!black, font=\footnotesize, above=2pt] at (7.3,1.4)
    {stopped at $\tau$};
\end{tikzpicture}
\caption{Schematic feasible region for anytime-valid inference.
  An e-process $E_t$ starts at $E_0=1$ and satisfies the stopped-value
  condition of Definition~\ref{def:av-class}; the shaded region
  depicts $\cC_{\mathrm{AV}}$.  For a point null, the solid line
  represents an admissible e-process: a \emph{nonnegative
  martingale} (Theorem~\ref{thm:ramdas}).  The dotted line is a
  process drawn to violate the stopped-value condition under some
  $\Pbb\in\cH_0$, and is therefore infeasible.  Stopping at any
  data-dependent time $\tau$
  preserves type-I error control at level $\alpha$ via Ville's
  inequality: the stopped value $E_\tau\le 1/\alpha$ with
  probability at least $1-\alpha$ (downward arrow).  All
  distinctions in the figure are encoded by line style so that the
  caption reads correctly in grayscale.}
\label{fig:av-cone-detail}
\end{figure}

\begin{proposition}[Martingale as structural bridge]%
\label{prop:bridge}
The martingale property relates to each admissibility criterion as
follows:
\begin{enumerate}[label=(\roman*),nosep]
\item Single-prior sequential Bayes $\Rightarrow$ martingale under
  the prior predictive: every posterior predictive sequence under
  a fixed prior $\Pi$ is a martingale under $\tilde P=\int P_\theta\,
  \dx\Pi(\theta)$ (Proposition~\ref{prop:bayes-martingale}).
\item For a point null, AV-admissible $\Leftrightarrow$ nonnegative
  martingale within $\cC_{\mathrm{AV}}$
  (Theorem~\ref{thm:ramdas}); for a composite null, the martingale
  condition under every null remains sufficient but is not necessary.
\item Coverage validity does not require any martingale property:
  conformal prediction sets are constructed from rank statistics.
\item Self-consistency $\not\Rightarrow$ Blackwell admissible: the
  plug-in $\hat p_n^{\mathrm{pi}}$ is a martingale under its own
  predictive law $\hat P$ but is strictly dominated
  (Theorem~\ref{thm:insuff}).
\end{enumerate}
For a point null, the martingale property is a complete
characterization of AV-admissibility within $\cC_{\mathrm{AV}}$,
and it is a structural feature of
single-prior sequential Bayes coherence, but it is not a universal
determinant of admissibility: it is neither necessary for fixed-sample
Blackwell admissibility, nor needed for coverage validity or
CApp boundary-feasibility.
\end{proposition}

\begin{remark}[Measure-relative martingale properties]\label{rem:measure-relative}
The martingale property of a sequence $(M_n)$ is defined relative to
a measure class, and the following three notions must be carefully
distinguished:
\begin{enumerate}[label=(\roman*),nosep]
\item \emph{Martingale under the prior predictive}
  $\tilde P=\int P_\theta\,\dx\Pi(\theta)$: the Bayesian posterior
  predictive sequence satisfies this by the tower property
  (Proposition~\ref{prop:bayes-martingale}).
\item \emph{Stopped-value control under each $P\in\cH_0$}: the defining
  property of e-processes in $\cC_{\mathrm{AV}}$
  (Definition~\ref{def:av-class}), which guarantees anytime-valid
  error control via Ville's inequality.
\item \emph{Self-predictive martingale under plug-in law $\hat P$}:
  the plug-in MLE $\hat p_n^{\mathrm{pi}}$ satisfies the martingale
  condition under $\hat P$ where
  $X_t\mid X_{1:t-1}\sim\mathrm{Bern}(\hat p_{t-1}^{\mathrm{pi}})$,
  yet is not admissible (Theorem~\ref{thm:insuff}).
\end{enumerate}
Martingale coherence is not invariant across measures; admissibility
depends on which measure class defines the risk functional.
\end{remark}

\begin{remark}[Certificate transport across filtrations]%
\label{rem:filtration-transport}
A related transport question arises \emph{inside} the anytime-valid
geometry itself.  \citet{choe2026filtrations} study $e$-processes
constructed in different filtrations and show that validity in a
coarser filtration does not automatically carry over to a finer
filtration; an \emph{adjuster} is needed to lift the evidence
  process, and they prove that such adjusters are in a precise sense
necessary.  Their result concerns transport across filtrations
within $\cC_{\mathrm{AV}}$; ours concerns transport across
criteria.  Both dependencies must be checked explicitly.
\end{remark}

\subsection{Marginal coverage admissibility}

The third admissibility geometry concerns prediction sets rather
than point predictions or sequential tests; its symmetry
assumption is exchangeability \citep{definetti1937,hewitt1955}.
Two notions of coverage must be distinguished.  \emph{Marginal coverage} averages
over both the calibration data and the new test point: it asks
whether the prediction set contains the true response with
probability at least $1-\alpha$ under the joint exchangeable
distribution.  \emph{Conditional coverage} conditions on the
observed test covariate $X_{n+1}=x$ and demands coverage at every
$x$ individually.  Marginal coverage is achievable by conformal
methods; conditional coverage, as the next theorem shows, is not
achievable without degenerate prediction sets.

\begin{definition}[Marginal coverage]\label{def:coverage}
Given exchangeable pairs
$(X_1,Y_1),\ldots,(X_{n+1},Y_{n+1})$, a prediction set
$\hat C_n(X_{n+1})$ satisfies \emph{marginal coverage at level
$1-\alpha$} if
\[
  \Pbb\bigl(Y_{n+1}\in\hat C_n(X_{n+1})\bigr)\ge 1-\alpha.
\]
The coverage condition by itself is a feasibility constraint, not
an admissibility relation: the trivial set $\hat C_n\equiv\cY$ always
satisfies it.  We promote it to a genuine dominance order by adding
expected length to the comparison.
\end{definition}

\begin{definition}[Coverage admissibility at level $1-\alpha$]%
\label{def:cov-adm}
Fix a coverage level $1-\alpha\in(0,1)$ and a declared class
$\mathcal P_{\mathrm{ex}}$ of exchangeable laws (a singleton is
allowed), together with a declared comparison class $\cG$ of
prediction-set processes.  The
\emph{coverage-feasible class} is
\[
  \cF_{\mathrm{Cov}}(\alpha)
  \;=\;
  \Bigl\{\hat C=(\hat C_n)\,:\,
    P\bigl(Y_{n+1}\in\hat C_n(X_{n+1})\bigr)\ge 1-\alpha
    \text{ for all }n\ge 1\text{ and }P\in\mathcal P_{\mathrm{ex}}\Bigr\}.
\]
For $\hat C,\hat C'\in\cF_{\mathrm{Cov}}(\alpha)\cap\cG$, write
$\hat C'\preceq_{\mathrm{Cov}}\hat C$ if
\[
  \Ebb_P|\hat C_n'|\;\le\;\Ebb_P|\hat C_n|
  \qquad\text{for every }n\ge 1
  \text{ and }P\in\mathcal P_{\mathrm{ex}},
\]
with strict inequality for some $(n,P)$.  A prediction-set process
$\hat C\in\cF_{\mathrm{Cov}}(\alpha)\cap\cG$ is
\emph{coverage-admissible within $\cG$}
\textup{(at level $1-\alpha$)} if no
$\hat C'\in\cF_{\mathrm{Cov}}(\alpha)\cap\cG$ satisfies
$\hat C'\preceq_{\mathrm{Cov}}\hat C$.
\end{definition}

\begin{remark}[Why fixed-level dominance, not a 2D Pareto order]%
\label{rem:trivial-set}
Coverage at level $1-\alpha$ is the validity constraint imposed by
the practitioner; once it is met, overcoverage carries no benefit,
and a 2D Pareto comparison on $(c,\ell)=(\Pbb(Y\in\hat C),\Ebb|\hat C|)$
would prevent any shorter valid set from dominating a wider one
(since the wider one typically has larger $c$).  We therefore
treat the coverage threshold as a feasibility requirement and
measure efficiency \emph{within} the feasible class by expected
length alone.  Under
Definition~\ref{def:cov-adm}, if the declared class $\cG$ contains
the trivial choice $\hat C_n\equiv\cY$, that choice has
coverage~$1\ge 1-\alpha$ but maximal expected length.  It is
strictly dominated whenever $\cG$ contains a feasible process that
is no wider for every $(n,P)$ and strictly shorter for at least one.
This turns marginal coverage from a one-sided feasibility condition
into an admissibility order.  An equivalent
``capped-coverage'' formulation, in which the coverage credit is
$\min\{\Pbb(Y\in\hat C),1-\alpha\}$, induces the same order on
$\cF_{\mathrm{Cov}}(\alpha)\cap\cG$.
\end{remark}

\begin{proposition}[Conformal sets are coverage-admissible within a rank-indexed family]%
\label{prop:conf-admissible}
Fix a nonconformity score fitted on an independent training split.
At each $n$, let $m_n$ calibration scores and the test score be
exchangeable.  Attach independent marks
$U_{n,1},\ldots,U_{n,m_n},U_{n,*}\sim\mathrm{Unif}(0,1)$,
independent of the scores.  Replace each calibration score
$s_{n,i}$ by $\tilde s_{n,i}=(s_{n,i},U_{n,i})$, and for a candidate
response $y$ use $\tilde s_n(x,y)=(s(x,y),U_{n,*})$, with pairs
ordered lexicographically.  Let $\tilde q_{n,k}$ be the $k$th order
statistic of the augmented calibration scores together with
$(+\infty,1)$, and define the randomized set
\[
  \hat C_n^{(k)}(x)
  =\{y:\tilde s_n(x,y)\le_{\mathrm{lex}}\tilde q_{n,k}\},
  \qquad k\in\{1,\ldots,m_n+1\}.
\]
The same marks are used for every rank $k$, so these randomized sets
are coupled and nested pathwise.
Let $\cG_s^{\mathrm{rank}}$ be the class of processes selecting one
such rank according to a deterministic sequence
$k_n\in\{1,\ldots,m_n+1\}$, with $k_n$ fixed before the calibration
scores, test score, and auxiliary marks are observed.  Take
$\mathcal P_{\mathrm{ex}}$ to be
the class of all laws under which the $m_n$ calibration scores and
the test score are exchangeable at every $n$.  Then the
split-conformal choice
\[
  k_{n,\alpha}=\left\lceil(m_n+1)(1-\alpha)\right\rceil
\]
is coverage-admissible at level $1-\alpha$ within
$\cG_s^{\mathrm{rank}}$.
\end{proposition}

\begin{remark}[Coverage admissibility is comparison-class relative]%
\label{rem:cov-class}
Coverage admissibility, like every efficiency
notion, is relative to a specified comparison class $\cG$ of
prediction-set processes.  The natural choice in conformal
practice is the rank-indexed family $\cG_s^{\mathrm{rank}}$ used in
Proposition~\ref{prop:conf-admissible}.  Throughout this paper,
when we say a system $\Delta\in\Sigma$ lies in
$\mathfrak{C}$, we mean its prediction-set coordinate is
coverage-admissible within a stated comparison class
(by default $\cG_s^{\mathrm{rank}}$ for the declared nonconformity
score $s$); the separation witnesses below keep this family
constant across the systems being compared.
\end{remark}

\begin{proof}
The independent continuous marks make the rank of the augmented
test score uniform on
$\{1,\ldots,m_n+1\}$ under every exchangeable score law.  Thus rank
$k$ has coverage $k/(m_n+1)$, and $k_{n,\alpha}$ is the smallest
rank with coverage at least $1-\alpha$.  Any process choosing a
smaller rank at some $n$ leaves the distribution-free feasible
class.  Any feasible process in $\cG_s^{\mathrm{rank}}$ therefore
chooses $k_n\ge k_{n,\alpha}$ for every $n$, and nesting gives
$\hat C_n^{(k_{n,\alpha})}\subseteq\hat C_n^{(k_n)}$ pathwise.
It cannot have strictly smaller expected size.  Hence no member of
the rank-indexed family dominates the conformal choice.
\end{proof}

\begin{theorem}[{\citealt{foygel2021}}]%
\label{thm:conformal-impossibility}
In the distribution-free setting, suppose a prediction set guarantees
conditional coverage
\[
  \Pbb\bigl(Y_{n+1}\in\hat C_n(X_{n+1})\mid X_{n+1}=x\bigr)\ge 1-\alpha
\]
for $P_X$-almost every $x$, uniformly over all distributions $P$.
Then, for every $P$, the set has infinite expected length at
$P_X$-almost every nonatomic $x$.  In particular, if $P_X$ is
nonatomic, it is essentially uninformative $P_X$-almost everywhere.
\end{theorem}

The intuition is that a distribution-free conditional guarantee over
essentially the entire covariate distribution would require the
prediction set to accommodate worst-case conditional laws throughout
that distribution; for nonatomic covariates this forces infinite
expected length almost everywhere.
The impossibility is specifically about \emph{distribution-free
conditional} coverage: it does not constrain marginal coverage,
which conformal methods do attain, and it does not by itself
adjudicate Blackwell optimality of a point predictor.  It is one
concrete face of the criterion separation that
Theorem~\ref{thm:separation} formalizes;
Figure~\ref{fig:coverage-detail} summarizes the coverage--width
trade-off.

\begin{figure}[ht]
\centering
\begin{tikzpicture}[>=Stealth, scale=1.0]
  \draw[->, thick] (-0.4,0) -- (8.5,0)
    node[below=4pt, font=\small] {expected interval width};
  \draw[->, thick] (0,-0.4) -- (0,6.0)
    node[left=4pt, font=\small] {marginal coverage};

  \fill[orange!8] (0.0,3.8) rectangle (8.0,5.6);
  \draw[orange!70!black, thick] (0.0,3.8) -- (8.0,3.8);
  \node[orange!70!black, font=\footnotesize, anchor=west]
    at (8.05,3.8) {$1{-}\alpha$};

  \draw[blue!60!black, very thick]
    (0.3,0.3) .. controls (1.4,1.7) and (2.5,3.0) .. (3.6,3.8)
    .. controls (4.7,4.5) and (6.0,5.0) .. (7.6,5.3);

  \fill[black]         (3.6,3.8) circle (2.7pt);
  \node[font=\footnotesize, anchor=north west] at (3.7,3.7) {C};
  \fill[red!70!black]  (2.0,2.5) circle (2.7pt);
  \node[red!70!black, font=\footnotesize, anchor=east]
    at (1.88,2.5) {O};
  \fill[green!50!black] (6.3,5.0) circle (2.7pt);
  \node[green!50!black, font=\footnotesize, anchor=south west]
    at (6.4,5.0) {W};
\end{tikzpicture}

\smallskip
{\footnotesize
\textbf{Legend.}
The shaded band above the horizontal line at $1{-}\alpha$ is the
coverage-feasible class $\cF_{\mathrm{Cov}}(\alpha)$
(Definition~\ref{def:cov-adm}); the curve is the
achievable coverage as a function of expected interval width.
\textbf{C}~$=\hat C_n^{\mathrm{conf}}$: the split-conformal set at
level $1-\alpha$, a representative exchangeability-based set
meeting at least the target coverage and lying on the rank-family
frontier for expected length
(Proposition~\ref{prop:conf-admissible}).
\textbf{O}~$=\hat C_n^{\mathrm{oracle}}$: a shorter oracle Bayes
interval that undercovers and so falls outside
$\cF_{\mathrm{Cov}}(\alpha)$.
\textbf{W}~$=\hat C_n^{\mathrm{wide}}$: a conservative feasible
set that wastes width.  Exact conditional coverage at every $x$
simultaneously is impossible
(Theorem~\ref{thm:conformal-impossibility}).}
\caption{Coverage-feasible region for prediction sets.  The
horizontal line marks the validity threshold $1-\alpha$; the
shaded band above is the feasible class
$\cF_{\mathrm{Cov}}(\alpha)$.  Among the three points,
\textbf{C}~lies on the rank-family frontier (conformal and meeting
at least the target coverage), \textbf{O}~lies below it (oracle Bayes,
infeasible), and \textbf{W}~lies above it (conservative,
feasible but wider than \textbf{C}).}
\label{fig:coverage-detail}
\end{figure}

\subsection{Criterion-relative admissibility}

\begin{definition}[Criterion-relative admissibility]\label{def:crit-adm}
Write $\mathsf{Adm}_{\cC}(\delta)$ to denote that $\delta$ is
admissible relative to criterion $\cC$: a comparison class of
procedures together with the partial order induced by the relevant
performance functional.  The three criteria considered are:
\begin{enumerate}[label=(\roman*),nosep]
\item $\cC_B$: Blackwell risk dominance: comparison class $\cD$,
  ordering by coordinatewise risk dominance on $\cR$, ambient
  geometry the convex risk set $\cR\subset\Rbb^k_+$.
\item $\cC_{\mathrm{AV}}$: anytime-valid admissibility: comparison
  class $\cC_{\mathrm{AV}}$ (Definition~\ref{def:av-class}), ordering
  by almost-sure pointwise dominance at every time under every null,
  ambient geometry the anytime-valid feasible class.
\item $\cC_{\mathrm{Cov}}$: marginal coverage at level $1-\alpha$:
  comparison class a specified rank-indexed family
  $\cG_s^{\mathrm{rank}}$ of prediction sets
  under exchangeability, ordering by expected length within the
  fixed-level coverage-feasible class
  $\cF_{\mathrm{Cov}}(\alpha)=\{\hat C:\Pbb(Y_{n+1}\in\hat C)\ge 1-\alpha\}$,
  ambient geometry the length frontier of
  $\cF_{\mathrm{Cov}}(\alpha)\cap\cG_s^{\mathrm{rank}}$.
\end{enumerate}
For a system $\Delta=(\delta,E,\hat C,\sigma)\in\Sigma$, membership
in $\mathfrak{B}$, $\mathfrak{A}$, and $\mathfrak{C}$ means
$\mathsf{Adm}_{\cC_B}(\delta)$,
$\mathsf{Adm}_{\cC_{\mathrm{AV}}}(E)$, and
$\mathsf{Adm}_{\cC_{\mathrm{Cov}}}(\hat C)$, respectively.
\end{definition}

\subsection{Constrained Bayes as a design principle}\label{sec:constrained}

Each criterion in Definition~\ref{def:crit-adm} can be written as a
constrained optimization with its own decision and feasibility
spaces.

\begin{definition}[Constrained Bayes schema]\label{def:cb}
Each criterion $C\in\{B,\mathrm{AV},\mathrm{Cov},\mathrm{App}\}$
carries its own criterion-specific decision space $\cD_C$ (the
class of objects on which $C$ is defined: point predictors for
Blackwell, $e$-processes for AV, prediction sets for Cov, online
strategies for App), its own criterion-specific risk functional
$\cR_C(\theta,\cdot)$, and its own feasibility set
$\cF_C\subseteq\cD_C$ encoding the validity requirement of~$C$.
Given a prior $\Pi$ on $\Theta$, the \emph{constrained Bayes
problem under criterion~$C$} is
\begin{equation}\label{eq:cb}
  \min_{\delta\in\cD_C}\;
    \int_\Theta \cR_C(\theta,\delta)\,\dx\Pi(\theta)
  \qquad\text{subject to}\qquad
  \delta\in\cF_C.
\end{equation}
A solution $\delta^*_{\cF_C}$ is a \emph{constrained Bayes rule for
criterion~$C$}.  When $C=B$ and $\cF_B=\cD_B=\cD$ is the
unconstrained point-predictor space, the problem reduces to the
unconstrained Bayes problem of Theorem~\ref{thm:supporting}.

Let $\operatorname{pr}_C:\Sigma\to\cD_C$ denote the projection onto
the coordinate relevant to criterion $C$.  Equivalently, lifting
each component to $\Sigma$ (Definition~\ref{def:predictive-system}),
the schema is
\begin{equation*}
  \min_{\substack{\Delta\in\Sigma:\\
        \operatorname{pr}_C(\Delta)\in\cF_C}}\;
    \int_\Theta
      \cR_C\bigl(\theta,\operatorname{pr}_C(\Delta)\bigr)
      \,\dx\Pi(\theta).
\end{equation*}
The other coordinates are unrestricted.  Activeness is imposed
only when constructing the separation witnesses; it is not part of
the constrained Bayes feasibility condition.
\end{definition}

\begin{remark}[The schema is not a single optimization problem]%
\label{rem:schema-not-single}
The common form in~\eqref{eq:cb} does not identify the decision
spaces or risk functionals: point predictors, $e$-processes,
prediction sets, and online strategies retain their own feasible
sets and partial orders.
\end{remark}

\begin{remark}[How the prior $\Pi$ enters the schema]%
\label{rem:prior-sensitivity}
Within the constrained Bayes schema two roles for the prior $\Pi$
must be carefully distinguished.

\smallskip\noindent
\textit{(1) $\Pi$ as objective weight, $\cF_C$ held fixed.}
Holding the feasibility set $\cF_C$ fixed, the prior $\Pi$ enters
only the objective in~\eqref{eq:cb}; changing $\Pi$ does
\emph{not} change feasibility.  When the restricted risk set is
finite-dimensional, closed, convex, and finite-valued,
geometrically, $\Pi$ is proportional to a supporting-hyperplane normal
of $\cR_{\cF_C}$ (equivalently, a vector of shadow prices), and
changing it can select a different supported point or face of the
accessible frontier $\partial_-\cR_{\cF_C}$, exactly the role it
plays on the unconstrained $\partial_-\cR$.

\smallskip\noindent
\textit{(2) $\Pi$ also entering $\cF_C$.}
If $\Pi$ also appears in the definition of $\cF_C$ (for example,
through a Bayesian credible-set construction that prescribes a set
estimator built from the posterior under $\Pi$), then changing
$\Pi$ changes both the objective and the feasible set.  This is no
longer the same constrained Bayes problem but a different one;
the relationship between solutions across priors is then governed
by joint sensitivity of the objective and the constraint, not by
hyperplane rotation alone.

\smallskip
In the rest of the paper we work in regime (1) wherever possible:
the feasibility constraints
$\cF_{\mathrm{AV}},\cF_{\mathrm{Cov}},\cF_{\mathrm{App}}$ are
prior-free, and changing $\Pi$ can change the selected supported
solution or face without altering $\cF_C$.
\end{remark}

The four criterion geometries of this paper correspond to four
choices of $(\cD_C,\cF_C,\cR_C)$.  Under Blackwell admissibility,
$\cD_B=\cD$ (point predictors), $\cF_B=\cD$ (no constraint),
$\cR_B(\theta,\delta)=R(\theta,\delta)$; under the full-support
conditions of Corollary~\ref{cor:noshame} or
Corollary~\ref{cor:noshame-compact}, every Bayes solution lies on
$\partial_-\cR$.  For anytime-valid inference,
$\cD_{\mathrm{AV}}=\cC_{\mathrm{AV}}$ (the $e$-process class),
$\cF_{\mathrm{AV}}=\cC_{\mathrm{AV}}$ itself with the stopped-value
constraint built in, and $\cR_{\mathrm{AV}}$ may be an
analyst-specified alternative-side loss, such as negative expected
log growth.  For a point null, admissibility is characterized by the
nonnegative martingale property (Theorem~\ref{thm:ramdas}).  For marginal coverage,
$\cD_{\mathrm{Cov}}$ is the declared comparison class $\cG$ of
prediction-set processes,
$\cF_C=\cF_{\mathrm{Cov}}(\alpha)\cap\cG$ under
Definition~\ref{def:cov-adm}, and $\cR_{\mathrm{Cov}}$ is expected
set length within that level-$(1-\alpha)$ feasible class.
For CApp boundary-feasibility, $\cD_{\mathrm{App}}$ is the space of
online strategies, $\cF_{\mathrm{App}}$ is the set of
CApp-feasible strategies (Definition~\ref{def:capp}), and
$\cR_{\mathrm{App}}$ is the limsup of $\bar R_n - \partial_-\cR$.
The prior $\Pi$ enters each problem as the objective weight, not
as a constraint.

When a constraint $\cF_C\subsetneq\cD_C$ is imposed, the
optimization in~\eqref{eq:cb} takes place over the restricted risk
set $\cR_{\cF_C}=\{r_C(\delta):\delta\in\cF_C\}\subseteq\cR_C$.
The lower boundary of $\cR_{\cF_C}$ need not coincide with
$\partial_-\cR$:
a constrained Bayes rule may fail to be Blackwell admissible because
$\partial_-\cR_{\cF_C}$ may contain points in the interior of $\cR$,
and a Blackwell admissible rule may be infeasible because
$r(\delta_\Pi)$ may lie outside $\cR_{\cF_C}$.  Bayes remains the primal
objective in every instance; the feasibility constraint determines
which portion of the risk frontier is accessible, and thereby which
priors can be realized as supporting hyperplanes of $\cR_{\cF_C}$.

The constraint sets $\cF_B$, $\cF_{\mathrm{AV}}$,
$\cF_{\mathrm{Cov}}(\alpha)\cap\cG$, and $\cF_{\mathrm{App}}$ live on
different coordinates.  The separation theorems establish that
their partial orders do not subsume one another.  A system may
nevertheless combine components satisfying several criteria
(Section~\ref{sec:compat-size}).

\subsection{Separation theorem}

\paragraph{Intuition for Theorem~\ref{thm:separation}.}
Blackwell, AV, and coverage admissibility act on the $\delta$, $E$,
and $\hat C$ coordinates, respectively.  The theorem uses systems
with both relevant coordinates active and proves a dominance
failure in one of them; the result is not obtained by leaving a
coordinate absent.

\begin{theorem}[Criterion separation on $\Sigma$]\label{thm:separation}
Let $\mathfrak{B}$, $\mathfrak{A}$, $\mathfrak{C}\subseteq\Sigma$
denote the classes of predictive systems
(Definition~\ref{def:predictive-system}) that are Blackwell
admissible, anytime-valid admissible, and coverage-admissible,
respectively, in the sense of Definition~\ref{def:meta-adm}.  For
each ordered pair $(\mathfrak{X},\mathfrak{Y})\subseteq
\{\mathfrak{B},\mathfrak{A},\mathfrak{C}\}^2$ with
$\mathfrak{X}\ne\mathfrak{Y}$, there exists a predictive system
$\Delta\in\Sigma$ with components active in both the
$\mathfrak{X}$- and $\mathfrak{Y}$-coordinates such that
$\Delta\in\mathfrak{X}$ and $\Delta\notin\mathfrak{Y}$.  Thus the
classes remain pairwise non-nested when both relevant coordinates
are required to be active.
\end{theorem}

The table lists the six witnesses and their active membership and
failure coordinates.

\begin{center}
\small
\begin{tabular}{@{}lll@{}}
\toprule
Separation & Membership coordinate & Active failure coordinate\\
\midrule
$\mathfrak{B}\not\subseteq\mathfrak{A}$
  & Bayes predictor $\delta^B$
  & strict supermartingale $E^\lambda$\\
$\mathfrak{A}\not\subseteq\mathfrak{B}$
  & LR martingale $E^{\mathrm{LR}}$
  & constant-zero predictor $\delta\equiv 0$\\
$\mathfrak{B}\not\subseteq\mathfrak{C}$
  & Bayes predictor $\delta^B$
  & strict rank enlargement $\hat C^{\mathrm{wide}}$\\
$\mathfrak{C}\not\subseteq\mathfrak{B}$
  & conformal set $\hat C^{\mathrm{conf}}$
  & plug-in predictor $\delta^{\mathrm{pi}}$\\
$\mathfrak{A}\not\subseteq\mathfrak{C}$
  & LR martingale $E^{\mathrm{LR}}$
  & strict rank enlargement $\hat C^{\mathrm{wide}}$\\
$\mathfrak{C}\not\subseteq\mathfrak{A}$
  & conformal set $\hat C^{\mathrm{conf}}$
  & strict supermartingale $E^\lambda$\\
\bottomrule
\end{tabular}
\end{center}

\begin{proof}
Fix a compact parameter set $\Theta=[\eta,1-\eta]\subset(0,1)$
for some $\eta\in(0,\tfrac12)$ and the i.i.d.\ Bernoulli model
$X_i\overset{\mathrm{iid}}{\sim}\mathrm{Bern}(\theta)$ with
$\theta\in\Theta$, together with $\alpha\in(0,1)$ and
$\theta_0\in\Theta$.  For the coverage coordinate, use a constant
covariate and response $Y_i=X_i$, so the exchangeable pairs encode
the same Bernoulli experiment.  For the AV coordinate take the point null
$\cH_0=\{P_{\theta_0}\}$.  Let $\Pi_J$ denote the
$\mathrm{Beta}(\tfrac12,\tfrac12)$ (Jeffreys) prior
restricted and renormalized to $\Theta$, so $\Pi_J$ has full
topological support on $\Theta$.  Define:
\begin{itemize}[leftmargin=*,nosep]
\item $\delta_n^B=\Ebb_{\Pi_J}[\theta\mid X_1,\ldots,X_n]$: the
  Bayes posterior predictive under the Jeffreys prior restricted
  and renormalized to $\Theta$ (the unrestricted closed form
  $(S_n+\tfrac12)/(n+1)$ holds only on the full Bernoulli model
  $\Theta=(0,1)$; the truncated prior gives a different but still
  $\cF_n$-measurable posterior mean);
\item $\delta_n^{\mathrm{pi}}=S_n/n$: plug-in MLE;
\item $E_n^{\mathrm{LR}}=\prod_{t=1}^n
  (\delta_{t-1}^B/\theta_0)^{X_t}
  ((1-\delta_{t-1}^B)/(1-\theta_0))^{1-X_t}$: likelihood-ratio
  e-process testing $H_0:\theta=\theta_0$;
\item $\hat C_n^{\mathrm{conf}}=\hat C_n^{(k_{n,\alpha})}$: the
  split-conformal prediction set from
  Proposition~\ref{prop:conf-admissible}, using the fixed binary
  score $s(y)=y$ and randomized tie-breaking;
  choose an index $n_0$ and calibration size $m_{n_0}$ large enough
  that $k_{n_0,\alpha}\le m_{n_0}-1$;
\item $\hat C^{\mathrm{wide}}$: a strict rank enlargement in the
  same family $\cG_s^{\mathrm{rank}}$.  At that index set
  $k'_{n_0}=k_{n_0,\alpha}+1$ and
  $k'_n=k_{n,\alpha}$ otherwise.  Under every Bernoulli law with
  $\theta\in(0,1)$, the event that exactly $k_{n_0,\alpha}$ of the
  $m_{n_0}$ calibration responses are zero has positive probability.
  On this event the adjacent thresholds cross from score $0$ to
  score $1$, so the enlarged prediction set is strictly wider with
  positive randomized tie-breaking probability;
\item $\sigma^{\mathrm{act}}$: any fixed active strategy distinct
  from the default-inactive constant $a_\varnothing$, for instance
  the constant strategy $\sigma^{\mathrm{act}}_n\equiv a_1$ with
  $a_1\ne a_\varnothing$.  Defensive forecasting is not needed
  here and is not introduced until Section~\ref{sec:constructive}.
\end{itemize}
All four coordinates are active in every witness below.

\smallskip\noindent
\emph{(i) $\mathfrak{B}\not\subseteq\mathfrak{A}$.}  Take
$\Delta'_B=(\delta^B,E^\lambda,\hat C^{\mathrm{conf}},\sigma^{\mathrm{act}})$
with $E^\lambda_n:=\lambda^n E^{\mathrm{LR}}_n$ for any fixed
$\lambda\in(0,1)$.  Then $E^\lambda\ge 0$ and
$\Ebb[E^\lambda_{n+1}\mid\cF_n]=\lambda^{n+1}E^{\mathrm{LR}}_n=
\lambda E^\lambda_n<E^\lambda_n$ under every $P\in\cH_0$, so
$E^\lambda$ is a strict supermartingale: it lies in
$\cC_{\mathrm{AV}}$ but, not being a nonnegative martingale, is
inadmissible within $\cC_{\mathrm{AV}}$ by
Theorem~\ref{thm:ramdas} (it is strictly dominated by
  $E^{\mathrm{LR}}$).  Hence $\Delta'_B\notin\mathfrak{A}$.  The
predictor $\delta^B$ is Blackwell admissible by
Corollary~\ref{cor:noshame-compact} (compact~$\Theta$ with the
full-support prior $\Pi_J$ and continuous risk under log loss on
$\Theta$), so $\Delta'_B\in\mathfrak{B}$.  Both the predictor and
the $e$-process coordinates are active.

\smallskip\noindent
\emph{$\mathfrak{A}\not\subseteq\mathfrak{B}$.}  Take
$\Delta'_A=(\delta',E^{\mathrm{LR}},\hat C^{\mathrm{conf}},
\sigma^{\mathrm{act}})$ where $\delta'\equiv 0$ is the constant
zero predictor.  The predictor is feasible but
Blackwell-inadmissible (its log-loss risk is $+\infty$ at every
$\theta>0$, dominated by $\delta^B$); hence
$\Delta'_A\notin\mathfrak{B}$.  The e-process $E^{\mathrm{LR}}$ is
a nonnegative martingale under $H_0$, hence AV-admissible by
Theorem~\ref{thm:ramdas}, giving $\Delta'_A\in\mathfrak{A}$.  Both
predictor and $e$-process components are active (the constant
zero predictor is a definite choice, not the default-inactive
placeholder constant $a_\varnothing$ used to lift).

\smallskip\noindent
\emph{(ii) $\mathfrak{B}\not\subseteq\mathfrak{C}$.}  Modify
$\Delta'_B$ by replacing $\hat C^{\mathrm{conf}}$ with
$\hat C^{\mathrm{wide}}$.  The enlarged set remains in
$\cF_{\mathrm{Cov}}(\alpha)$ and is strictly dominated by
$\hat C^{\mathrm{conf}}$ within $\cG_s^{\mathrm{rank}}$, so the
modified system is not in $\mathfrak{C}$ by
Definition~\ref{def:cov-adm}.  Both the predictor and the set
coordinate are active within their respective comparison
classes.  The predictor remains $\delta^B$, so the system is in
$\mathfrak{B}$.

\smallskip\noindent
\emph{$\mathfrak{C}\not\subseteq\mathfrak{B}$.}  Retain
$\hat C^{\mathrm{conf}}$ (coverage-admissible by
Proposition~\ref{prop:conf-admissible}) and use the plug-in
predictor $\delta^{\mathrm{pi}}$, which is Blackwell-inadmissible
under log loss (Theorem~\ref{thm:insuff}).  The system
$(\delta^{\mathrm{pi}},E^{\mathrm{LR}},\hat C^{\mathrm{conf}},
\sigma^{\mathrm{act}})$ is therefore in
$\mathfrak{C}\setminus\mathfrak{B}$.

\smallskip\noindent
\emph{(iii) $\mathfrak{A}\not\subseteq\mathfrak{C}$.}  Take
$\Delta'_A$ above and replace its prediction-set component by the
strictly wider $\hat C^{\mathrm{wide}}$ from~(ii).  Then
$\Delta'_A\in\mathfrak{A}$ but its prediction-set component is not
coverage-admissible within $\cG_s^{\mathrm{rank}}$, so
$\Delta'_A\notin\mathfrak{C}$.

\smallskip\noindent
\emph{$\mathfrak{C}\not\subseteq\mathfrak{A}$.}  Modify
$(\delta^{\mathrm{pi}},E^{\mathrm{LR}},\hat C^{\mathrm{conf}},
\sigma^{\mathrm{act}})$ by replacing $E^{\mathrm{LR}}$ with the
strict-supermartingale process
$E^\lambda_n=\lambda^n E^{\mathrm{LR}}_n$ of~(i).
The prediction-set component remains coverage-admissible; the
e-process is no longer AV-admissible.  The system is in
$\mathfrak{C}\setminus\mathfrak{A}$.
\end{proof}

\paragraph{Interpretation.}
Within $\Sigma$, none of the three criterion orders subsumes the
other two.  A product or lexicographic order could rank all
coordinates, but that order would add a choice not contained in
the three criteria.

\subsection{Compatibility of the admissible classes}
\label{sec:compat-size}

The classes are non-nested, not disjoint.  The following systems
give three nonempty intersections.

\begin{enumerate}[label=\textup{(\roman*)},leftmargin=*,nosep]
\item A predictive system whose $e$-process is the likelihood ratio
   under the Bayes posterior predictive and whose prediction set is
   the split-conformal set built from the same predictor lies in
   $\mathfrak{A}\cap\mathfrak{C}$: it controls type-I error at every
   stopping time and is expected-length efficient within its
   rank-indexed family in $\cF_{\mathrm{Cov}}(\alpha)$
   (Proposition~\ref{prop:conf-admissible}).
\item A Bayes predictor under a full-support prior, wrapped by a
   split-conformal procedure, lies in $\mathfrak{B}\cap\mathfrak{C}$:
   admissibility of the predictor coordinate is supplied by
   Corollary~\ref{cor:noshame}; admissibility of the set coordinate
   is supplied by Proposition~\ref{prop:conf-admissible}.
\item On a compact Bernoulli submodel
   $\Theta=[\eta,1-\eta]\subset(0,1)$, the Bayes posterior
   predictive under a smooth full-support prior such as the
   Jeffreys prior restricted and renormalized to $\Theta$ lies
   in $\mathfrak{B}$ by
   Corollary~\ref{cor:noshame-compact}.  Standard log-loss
   redundancy bounds for regular one-dimensional Bernoulli
   mixtures \citep[\S9.7]{cesabianchi2006} give Ces\`{a}ro
   convergence of its time-averaged risk to $H(\theta)$, hence
   CApp boundary-feasibility, so
   the lifted system is also in $\mathfrak{D}$.  The KT formula
   $\hat p_n^{\mathrm{KT}}=(S_n+\tfrac12)/(n+1)$ provides the
   simplest explicit version of the same phenomenon on the
   full Bernoulli model.  Thus
   $\mathfrak{B}\cap\mathfrak{D}\neq\varnothing$; constructive
   and Ces\`{a}ro admissibility \emph{can} coincide.
\end{enumerate}

No cardinality or measure comparison among the four classes is
asserted.  The examples establish nonempty intersections, while
Theorems~\ref{thm:separation} and~\ref{thm:extended-separation}
establish pairwise non-nesting.  Since the criteria use different
coordinates and orders, terms such as ``larger'' or ``thinner'' are
not invariantly defined across the four classes.

\section{Constructive and Choice-Based Admissibility}%
\label{sec:constructive}

The connection between Blackwell approachability
\citep{blackwell1956} and no-regret learning
\citep{hannan1957,abernethy2011} invites a question: does steering
the time-averaged risk to the lower boundary suffice for
admissibility?  This section distinguishes two routes to
$\partial_-\cR$ and shows that the distinction generates a fourth
admissibility geometry.

\subsection{Two paths to the boundary}\label{sec:two-paths}

\begin{definition}[Constructive admissibility]\label{def:constructive}
A procedure $\delta$ is \emph{constructively admissible} if there
exists a prior $\Pi_n$ at every sample size $n$ such that
$\delta_n(X^n)=\delta_{\Pi_n}(X^n)$ and $\delta_n$ is Blackwell
admissible.  Thus every round has both an explicit Bayes witness and
a boundary certificate; under the full-support conditions of
Corollary~\ref{cor:noshame} or~\ref{cor:noshame-compact}, the latter
follows automatically from the former.
\end{definition}

This is a pointwise (per-round) boundary condition: each action is
itself boundary-certified and has an explicit supporting prior.

\begin{definition}[Ces\`{a}ro admissibility]\label{def:cesaro}
A procedure $\delta$ is \emph{Ces\`{a}ro admissible} if the
time-averaged risk
$\bar R_n(\theta,\delta)=n^{-1}\sum_{t=1}^n R(\theta,\delta_t)$
converges to the oracle envelope $\partial_-\cR(\theta)$ as
$n\to\infty$ (the per-$\theta$ envelope sense made precise in
Definition~\ref{def:capp}), without requiring
that each individual action $\delta_t$ be Bayes optimal at time~$t$.
\end{definition}

This is a Ces\`{a}ro boundary condition: only the time-average risk
approaches the envelope $\partial_-\cR(\theta)$, so individual rounds
may lie in the interior of the risk set.

Constructive admissibility demands a per-round witness (the prior
$\Pi_n$), while Ces\`{a}ro admissibility requires only that the
long-run average reaches the boundary.  The Blackwell
approachability theorem guarantees the existence of
Ces\`{a}ro-admissible strategies whenever the target set is
approachable; it does not, in general, produce constructively
admissible ones.

\subsection{Approachability and single-prior Bayes coherence}%
\label{sec:abh-revisited}

\citet{abernethy2011} show that Blackwell approachability and
no-regret learning are equivalent.  Approachability gives a
Ces\`{a}ro guarantee and does not require each action to be Bayes at
that round.  By contrast, posterior predictive means generated
from one fixed prior form a martingale under the corresponding
prior predictive law
(Proposition~\ref{prop:constructive-martingale-single}).  This
single-prior coherence condition is distinct from approachability.

The asymptotic calibration theorem of \citet{foster1998} instead
selects actions by a fixed-point argument (Kakutani's theorem, or
the continuous-selection refinements of \citet{hart2001}).  It
illustrates choice-based construction under a calibration-error
objective.

\subsection{Ces\`{a}ro does not imply pointwise}%
\label{sec:cesaro-gap}

\begin{remark}[Choice-based strategies illustrate Ces\`{a}ro guarantees]%
\label{rem:cesaro-not-pointwise}
Foster--Vohra calibration and defensive forecasting
\citep{foster1998,vovk2005defensive} use fixed-point selection
rather than posterior updating.  Their guarantee concerns
calibration error, not time-averaged log loss
(Remark~\ref{rem:df-vs-logloss}).
\end{remark}

\subsection{Martingale characterization under a single prior}%
\label{sec:constructive-martingale}

\begin{proposition}[Single-prior martingale characterization]%
\label{prop:constructive-martingale-single}
Let $\delta=(\delta_n)$ be Bayes under a single prior $\Pi$ at
every $n$, i.e.\ $\delta_n=\delta_\Pi$.  Then the posterior
predictive sequence $(m_n)=(\Ebb_\Pi[\theta\mid\cF_n])$ is a
martingale under the prior predictive law
$\tilde P=\int P_\theta\,\dx\Pi(\theta)$.
\end{proposition}

\begin{proof}
Immediate from Proposition~\ref{prop:bayes-martingale} applied to
$\tilde P$.
\end{proof}

\begin{remark}[Single prior vs.\ varying priors]\label{rem:single-vs-vary}
Only the single-prior result is used below; it follows directly
from Proposition~\ref{prop:bayes-martingale}.  If the prior $\Pi_n$
depends on sample size, being Bayes at each $n$ does \emph{not}
imply the existence of a single prior predictive law under which
the sequence is a martingale.
Projective compatibility of the data-marginal laws
$\tilde P_n=\int P_\theta^{(n)}\,\dx\Pi_n(\theta)$ is not sufficient:
it does not preserve the latent parameter whose posterior mean is
being compared.  Projective compatibility of the joint laws
$\Pi_n(\dx\theta)P_\theta^{(n)}(\dx x_{1:n})$ would suffice by the
tower property, but it also forces their $\theta$-marginals to
agree and thus restores a common latent prior.  The proposition is
therefore stated for a fixed prior.
\end{remark}

Thus the martingale characterization applies to single-prior
sequential Bayes coherence, whereas CApp requires only asymptotic
convergence to the oracle envelope.

\subsection{A fourth geometry: choice-based approachability admissibility}%
\label{sec:fourth-geometry}

\begin{definition}[CApp boundary-feasibility]\label{def:capp}
Fix a loss function $L$ and let
\[
  \bar R_n(\theta,\sigma)
  \;=\; n^{-1}\sum_{t=1}^n L\bigl(\theta,\sigma_t(X_{1:t-1})\bigr)
\]
denote the time-averaged risk of an online strategy $\sigma$.
The strategy $\sigma$ is \emph{CApp boundary-feasible} (or simply
\emph{CApp-feasible}) if
\[
  \mathrm{dist}\bigl(\bar R_n(\theta,\sigma),
       \partial_-\cR(\theta)\bigr)
  \;\longrightarrow\; 0
  \qquad\text{in $P_\theta$-probability as $n\to\infty$,
       for every $\theta\in\Theta$,}
\]
where $\partial_-\cR(\theta)$ denotes the \emph{oracle per-round
envelope} at $\theta$,
$\partial_-\cR(\theta)=\mathrm{env}(\theta):=\inf_{a\in\cA}L(\theta,a)$,
which equals the binary entropy $H(\theta)$ under Bernoulli log
loss.  We write $\mathfrak{D}$ for the class of
predictive systems whose strategy coordinate is CApp
boundary-feasible.  CApp boundary-feasibility is a coordinatewise
asymptotic statement, one $\theta$ at a time; it is not the
assertion that some finite-sample risk \emph{vector} lies on the
Pareto frontier $\partial_-\cR$ of Section~\ref{sec:geometry}, and
the two should not be conflated (cf.\ the envelope/boundary
distinction in Example~\ref{ex:bern-running}).
\end{definition}

\begin{remark}[Why feasibility, not full admissibility]%
\label{rem:capp-pins}
Definition~\ref{def:capp} specifies both the loss $L$, which is part
of the decision problem, and convergence in $P_\theta$-probability,
which accommodates stochastic strategies.
We stop at \emph{feasibility} (boundary-reachability), because the
separation theorems need no further dominance order among
CApp-feasible strategies.  Global mean matching alone does not
imply CApp boundary-feasibility or time-averaged log-loss
optimality (Remark~\ref{rem:df-vs-logloss}).
\end{remark}

For calibration-error objectives, defensive forecasting and
Foster--Vohra use choice-based fixed-point arguments
\citep{foster1998,vovk2005defensive,hart2001}.  For log loss, the
witness below is the Krichevsky--Trofimov (KT) forecaster.

\begin{lemma}[KT is CApp-feasible under log loss]\label{lem:df-boundary}
In the Bernoulli i.i.d.\ model under log loss, the Krichevsky--Trofimov
forecaster \citep{krichevsky1981}
$\hat p_n^{\mathrm{KT}}=(S_n+\tfrac12)/(n+1)$
(equivalently, the Bayes posterior predictive under the Jeffreys
prior $\mathrm{Beta}(\tfrac12,\tfrac12)$) satisfies
$\bar R_n(\theta,\hat p^{\mathrm{KT}})\to H(\theta)$ in
$P_\theta$-probability for every $\theta\in(0,1)$, where
$H(\theta)=-\theta\log\theta-(1-\theta)\log(1-\theta)$ is the
binary entropy.  Since $\partial_-\cR(\theta)=H(\theta)$ under log
loss, the KT forecaster is CApp-feasible.
\end{lemma}

\begin{proof}[Proof sketch]
The KT forecaster has cumulative log-loss regret of order
$O(\log n)$ relative to the best constant predictor in hindsight
\citep[\S9.7]{cesabianchi2006}.  By the law of large numbers, the
empirical best constant in hindsight converges $P_\theta$-a.s.\ to
$\theta$, and its average log loss converges to $H(\theta)$.  The
cumulative martingale-difference term between the realized log
loss at step $t$ and the conditional expected log loss given
$\cF_{t-1}$ is $o_p(n)$, hence its Ces\`{a}ro average is
$o_p(1)$.  Combining these three pieces, the time-averaged
conditional log-loss risk $\bar R_n(\theta,\hat p^{\mathrm{KT}})$
converges in $P_\theta$-probability to $H(\theta)$.
\end{proof}

\begin{remark}[Defensive forecasting: calibration vs.\ log-loss
boundary]\label{rem:df-vs-logloss}
Defensive forecasting yields empirical calibration; in
particular, its global mean mismatch may vanish.  This does not
imply time-averaged log-loss optimality because averaging
predictions and averaging a strictly convex loss are different
operations.  The log-loss CApp witness used here is KT
(Lemma~\ref{lem:df-boundary}); Foster--Vohra and defensive
forecasting are mentioned only for calibration-error objectives.
\end{remark}

\begin{remark}[$\mathfrak{B}\cap\mathfrak{D}\ne\emptyset$ on a
compact Bernoulli submodel]\label{rem:kt-overlap}
On the compact Bernoulli submodel $\Theta=[\eta,1-\eta]$, the
Bayes posterior predictive under a smooth full-support prior
(such as the Jeffreys prior restricted and renormalized to
$\Theta$) is Blackwell admissible by
Corollary~\ref{cor:noshame-compact}.  Standard log-loss
redundancy bounds for regular one-dimensional Bernoulli mixtures
\citep[\S9.7]{cesabianchi2006} give Ces\`{a}ro convergence of its
time-averaged risk to $H(\theta)$, hence CApp boundary-feasibility
(cf.\ Lemma~\ref{lem:df-boundary}).  The KT forecaster
$\hat p_n^{\mathrm{KT}}=(S_n+\tfrac12)/(n+1)$ is the simplest
explicit version of the same phenomenon on the full Bernoulli
model.  Hence $\mathfrak{B}\cap\mathfrak{D}\ne\emptyset$,
which is consistent with the separation theorem below: the
non-nesting is a statement about \emph{neither} containing the
other, not about disjointness.  The separations in
Theorem~\ref{thm:extended-separation} require either a
misspecified prior (case~(i) below) or a CApp-feasible strategy
outside the Bayes class (case~(ii)).
\end{remark}

\paragraph{Intuition for Theorem~\ref{thm:extended-separation}.}
Two witnesses handle the pairs involving $\mathfrak{D}$.  A
dogmatic Bayes rule gives
$\mathfrak{B}\not\subseteq\mathfrak{D}$; pairing the KT strategy
with the plug-in predictor gives
$\mathfrak{D}\not\subseteq\mathfrak{B}$.  The remaining pairs use
the same active-coordinate construction.

\begin{theorem}[Extended separation on $\Sigma$]%
\label{thm:extended-separation}
Let $\mathfrak{B}$, $\mathfrak{A}$, $\mathfrak{C}$, $\mathfrak{D}$
denote the classes of predictive systems
$\Delta\in\Sigma$ (Definition~\ref{def:predictive-system}) that
satisfy the corresponding criterion property in
Definition~\ref{def:meta-adm}.  The four classes are pairwise
non-nested: for each ordered pair
$(\mathfrak{X},\mathfrak{Y})$ with $\mathfrak{X}\ne\mathfrak{Y}$,
there exists a predictive system $\Delta\in\Sigma$ with components
active in both the $\mathfrak{X}$- and
$\mathfrak{Y}$-coordinates such that $\Delta\in\mathfrak{X}$ and
$\Delta\notin\mathfrak{Y}$.
\end{theorem}

\begin{proof}
The pairwise non-nesting of $\mathfrak{B}$, $\mathfrak{A}$,
$\mathfrak{C}$ within $\Sigma$ is established in
Theorem~\ref{thm:separation}.  It remains to verify the six
separations involving $\mathfrak{D}$.  Retain the point null and the
rank-indexed conformal family fixed in that theorem.

\smallskip\noindent
\emph{(i) $\mathfrak{B}\not\subseteq\mathfrak{D}$: misspecified-prior
witness.}  Work on the same compact Bernoulli model
$\Theta=[\eta,1-\eta]$ used in Theorem~\ref{thm:separation},
chosen so that $0.3, 0.7\in\Theta$, and let $\Pi=\delta_{0.3}$
be the point mass at $\theta=0.3$.  The Bayes rule under $\Pi$
at every $n$ is the constant predictor $\delta_\Pi\equiv 0.3$
(the posterior remains $\delta_{0.3}$ for all data; the Bayes
act minimizing $R(0.3,a)$ over $a\in[0,1]$ under log loss is
$a=0.3$).  This predictor is Blackwell admissible by a direct
argument (rather than via Corollary~\ref{cor:noshame-compact},
which requires full topological support).  Indeed, the
lower-boundary value at $\theta=0.3$ is $H(0.3)$, attained by
the constant action $a=0.3$.  Any rule $\delta'$ satisfying
$R(0.3,\delta')\le R(0.3,\delta_\Pi)=H(0.3)$ must, by strict
properness of log loss \citep{gneiting2007}, predict $a=0.3$
$P_{0.3}$-a.s.  Since
$P_\theta$ and $P_{\theta'}$ are mutually absolutely continuous
on $\{0,1\}^n$ for every $\theta,\theta'\in(0,1)$ (Bernoulli
product measures with parameters in $(0,1)$ have a common
support), equality $P_{0.3}$-a.s.\ implies equality
$P_\theta$-a.s.\ for every $\theta\in\Theta$.  Hence $\delta'$
agrees with $\delta_\Pi$ on a set of $P_\theta$-measure one for
each $\theta$ and cannot strictly improve anywhere on $\Theta$.
No rule dominates $\delta_\Pi$, so $\delta_\Pi$ is Blackwell
admissible.

Compute the per-round risk at the misspecified parameter
$\theta^*=0.7$:
\begin{align*}
  R(0.7,\delta_\Pi)
  &= -0.7\log(0.3) - 0.3\log(0.7)\\
  &\approx 0.8428 + 0.1070 = 0.9498\ \text{nats}.
\end{align*}
The Ces\`{a}ro average $\bar R_n(0.7,\delta_\Pi)\equiv 0.9498$ for
all $n$ (the predictor never updates).  The lower boundary at
$\theta=0.7$ is $H(0.7)=-0.7\log(0.7)-0.3\log(0.3)\approx 0.6109$.
Hence $\bar R_n(0.7,\delta_\Pi) - H(0.7)\approx 0.339$ nats per
round, a gap that never closes, so
$|\bar R_n(0.7,\delta_\Pi)-H(0.7)|\not\to 0$.
The associated predictive system
$\Delta=(\delta_\Pi,E^{\mathrm{LR}},\hat C^{\mathrm{conf}},\delta_\Pi)$
has active predictor, $e$-process, prediction-set, and strategy
components.  It satisfies $\Delta\in\mathfrak{B}$ (the predictor is
Blackwell admissible) yet $\Delta\notin\mathfrak{D}$ (the strategy
is not CApp-feasible).
The witness depends essentially on the dogmatic point mass, which
no amount of data corrects: a full-support prior would concentrate
at the truth and, by Remark~\ref{rem:kt-overlap}, land in
$\mathfrak{B}\cap\mathfrak{D}$.  (For the misspecified-model
analogue, where posteriors concentrate on the Kullback--Leibler
projection, see \citealt{kleijn2012}.)

\smallskip\noindent
\emph{(ii) $\mathfrak{D}\not\subseteq\mathfrak{B}$.}  Consider the
predictive system
\[
  \Delta \;=\;
  \bigl(\delta^{\mathrm{pi}},\,
        E^{\mathrm{LR}},\,
        \hat C^{\mathrm{conf}},\,
        \sigma^{\mathrm{KT}}\bigr),
\]
where $\sigma^{\mathrm{KT}}$ is the KT online strategy
$\sigma_n^{\mathrm{KT}}=(S_{n-1}+\tfrac12)/n$ (the Krichevsky--Trofimov
forecaster used in Lemma~\ref{lem:df-boundary}) and
$\delta^{\mathrm{pi}}_n=S_n/n$ is the plug-in MLE predictor.  The
strategy coordinate is CApp-feasible under log loss by
Lemma~\ref{lem:df-boundary}, so $\Delta\in\mathfrak{D}$ via
Definition~\ref{def:meta-adm}.  The predictor coordinate
$\delta^{\mathrm{pi}}$ is Blackwell-inadmissible under log loss by
Theorem~\ref{thm:insuff}, so $\Delta\notin\mathfrak{B}$.  Both the predictor and the
strategy coordinates are active: $\delta^{\mathrm{pi}}$ differs
from any constant predictor on a positive-measure event, and
$\sigma^{\mathrm{KT}}$ is not a constant strategy and differs from
the default on positive-probability events for $n\ge 2$.

\smallskip\noindent
\emph{(iii) $\mathfrak{A}\not\subseteq\mathfrak{D}$.}  Take the
explicit predictive system
$\Delta=(\delta^0,E^{\mathrm{LR}},\hat C^{\mathrm{conf}},\sigma^0)$
where $\delta^0\equiv 0$ and $\sigma^0\equiv 0$ are the constant
zero predictor and online strategy.  The e-process
$E^{\mathrm{LR}}$ is a nonnegative martingale under $\cH_0$, hence
AV-admissible (Theorem~\ref{thm:ramdas}), so
$\Delta\in\mathfrak{A}$.  The strategy $\sigma^0$ is not CApp
boundary-feasible under log loss: $R(\theta,0)=+\infty$ for every
$\theta\in(0,1)$, so $\bar R_n(\theta,\sigma^0)=+\infty\not\to
H(\theta)$.  Hence $\Delta\notin\mathfrak{D}$.

\smallskip\noindent
\emph{(iv) $\mathfrak{D}\not\subseteq\mathfrak{A}$.}  Take
$\Delta=(\delta^{\mathrm{pi}},E^\lambda,\hat C^{\mathrm{conf}},
\sigma^{\mathrm{KT}})$ with
$E^\lambda_n=\lambda^n E^{\mathrm{LR}}_n$, $\lambda\in(0,1)$.
The strategy $\sigma^{\mathrm{KT}}$ is CApp boundary-feasible by
Lemma~\ref{lem:df-boundary}, so $\Delta\in\mathfrak{D}$.  The
e-process $E^\lambda$ is a strict supermartingale in
$\cC_{\mathrm{AV}}$, hence inadmissible by Theorem~\ref{thm:ramdas}
(strictly dominated by $E^{\mathrm{LR}}$), so
$\Delta\notin\mathfrak{A}$.  (The constant unit process
$E\equiv 1$ is itself a nonnegative martingale and therefore
AV-admissible, which is why we use the strict supermartingale
$E^\lambda$ as the failure witness.)

\smallskip\noindent
\emph{(v) $\mathfrak{C}\not\subseteq\mathfrak{D}$.}  Take
$\Delta=(\delta^0,E^{\mathrm{LR}},\hat C^{\mathrm{conf}},\sigma^0)$
as in case~(iii).  The conformal set $\hat C^{\mathrm{conf}}$ is
coverage-admissible within $\cG_s^{\mathrm{rank}}$
(Proposition~\ref{prop:conf-admissible}), so
$\Delta\in\mathfrak{C}$; the strategy $\sigma^0\equiv 0$ has
$\bar R_n(\theta,\sigma^0)=+\infty$ under log loss, so
$\Delta\notin\mathfrak{D}$.

\smallskip\noindent
\emph{(vi) $\mathfrak{D}\not\subseteq\mathfrak{C}$.}  Take
$\Delta=(\delta^{\mathrm{pi}},E^{\mathrm{LR}},\hat C^{\mathrm{wide}},
\sigma^{\mathrm{KT}})$.  Then $\Delta\in\mathfrak{D}$ as in
case~(iv), while $\hat C^{\mathrm{wide}}$ is strictly dominated
within $\cG_s^{\mathrm{rank}}$ by $\hat C^{\mathrm{conf}}$ and is
therefore not coverage-admissible.  Both sets remain in
$\cF_{\mathrm{Cov}}(\alpha)$, and
$\Delta\notin\mathfrak{C}$.
\end{proof}

\subsection{Summary of the four criteria}%
\label{sec:four-interpretation}

Table~\ref{tab:taxonomy} summarizes the four criteria, their
witnesses, and their modes of optimality.

\begin{table}[ht]
\centering\small
\caption{Summary of the four criteria.}
\label{tab:taxonomy}
\begin{tabular}{@{}llll@{}}
\toprule
Geometry & Certificate & Boundary witness & Optimality mode\\
\midrule
$\mathfrak{B}$: Blackwell
  & supporting hyperplane & prior $\Pi$
  & pointwise (per-round)\\[3pt]
$\mathfrak{A}$: Anytime-valid
  & martingale (point null) & e-process
  & pathwise (all stopping times)\\[3pt]
$\mathfrak{C}$: Coverage
  & exchangeability rank & conformal set
  & rank-family expected length\\[3pt]
$\mathfrak{D}$: CApp
  & Ces\`{a}ro convergence & regret / redundancy
  & Ces\`{a}ro (time-averaged)\\
\bottomrule
\end{tabular}
\end{table}

The four-geometry diamond (Figure~\ref{fig:four-geometries})
illustrates the pairwise non-nesting.  Dashed double arrows connect
every pair, indicating non-nesting rather than inclusion, as established in
Theorems~\ref{thm:separation} and~\ref{thm:extended-separation}.

\begin{figure}[ht]
\centering
\resizebox{0.78\linewidth}{!}{%
\begin{tikzpicture}[>=Stealth, scale=1.0,
  geobox/.style={draw, thick, rounded corners=5pt,
    minimum width=4.2cm, minimum height=1.4cm, align=center,
    font=\footnotesize, fill=white},
  nnlabel/.style={fill=white, inner sep=1.5pt, font=\scriptsize,
    text=black!70}]
  \node[geobox] (B) at (0,3.6)
    {$\mathfrak{B}$: Blackwell\\[2pt]
     \scriptsize prior-witnessed, per-round};
  \node[geobox] (A) at (-4.2,0)
    {$\mathfrak{A}$: Anytime-valid\\[2pt]
     \scriptsize martingale-witnessed};
  \node[geobox] (C) at (4.2,0)
    {$\mathfrak{C}$: Coverage\\[2pt]
     \scriptsize exchangeability-witnessed};
  \node[geobox] (D) at (0,-3.6)
    {$\mathfrak{D}$: CApp\\[2pt]
     \scriptsize Ces\`{a}ro convergence via regret};

  \draw[<->, thick, black!55, dashed] (B.south west) -- (A.north east)
    node[midway, nnlabel] {non-nested};
  \draw[<->, thick, black!55, dashed] (B.south east) -- (C.north west)
    node[midway, nnlabel] {non-nested};
  \draw[<->, thick, black!55, dashed] (A.south east) -- (D.north west)
    node[midway, nnlabel] {non-nested};
  \draw[<->, thick, black!55, dashed] (C.south west) -- (D.north east)
    node[midway, nnlabel] {non-nested};
  \draw[<->, thick, black!55, dashed] (B.south) -- (D.north)
    node[midway, nnlabel] {non-nested};
  \draw[<->, thick, black!55, dashed] (A.east) -- (C.west)
    node[midway, nnlabel] {non-nested};
\end{tikzpicture}%
}
\caption{Four admissibility geometries in diamond configuration.
  Each node represents an admissible class; dashed arrows indicate
  pairwise non-nesting
  (Theorems~\ref{thm:separation}
  and~\ref{thm:extended-separation}).
  Blackwell admissibility and CApp boundary-feasibility both involve
  risk but differ in witness type (prior support vs.\ Ces\`{a}ro
  convergence); anytime-valid
  and coverage admissibility operate on different procedure spaces
  entirely.}
\label{fig:four-geometries}
\end{figure}

\section{Bernoulli Laboratory}\label{sec:bernoulli}

The Bernoulli model permits closed-form examples for all four
criteria.  This section collects the four procedures used in the
separation results.

Table~\ref{tab:bernoulli} records the four procedures and four
criterion classes.  N/A indicates that the criterion is defined
for a categorically different class of procedures; the assessment is
inapplicable rather than negative.

\begin{table}[ht]
\centering
\footnotesize
\caption{Bernoulli laboratory: four procedures and four
  criterion classes.  Column headers: \emph{Martingale property}
  under the appropriate measure (Definition~\ref{def:three-coherence});
  \emph{Blackwell admissibility}; \emph{AV-admissibility} (in the
  sense of Theorem~\ref{thm:ramdas}); \emph{Coverage admissibility}
  (at level $1-\alpha$ within the declared rank family,
  Definition~\ref{def:cov-adm});
  \emph{CApp boundary-feasibility} (Definition~\ref{def:capp}).
  \checkmark{}~$=$~satisfies; $\times$~$=$~fails;
  N/A~$=$~criterion not applicable to this procedure type.
  The unclipped plug-in MLE assigns zero probability to realizable
  events whenever $S_n\in\{0,n\}$, giving infinite log loss; its
  Ces\`{a}ro-average risk does not converge to $H(\theta)$ in the
  extended-real sense, so it is not CApp boundary-feasible under
  log loss.  $^{\dagger}$Fixed clipping by $\epsilon>0$ removes the
  boundary blowup but destroys the plug-in martingale property and
  is not CApp-feasible on the full model $\theta\in(0,1)$; it is
  CApp-feasible only after restricting to
  $\theta\in[\epsilon,1-\epsilon]$.  We include the clipped row only
  as a numerical convention (Section~\ref{sec:numerical}).}
\label{tab:bernoulli}
\begin{tabular}{@{}lccccc@{}}
\toprule
Procedure
  & \begin{tabular}[c]{@{}c@{}}Martingale\\property\end{tabular}
  & \begin{tabular}[c]{@{}c@{}}Blackwell\\admissible\end{tabular}
  & \begin{tabular}[c]{@{}c@{}}AV-\\admissible\end{tabular}
  & \begin{tabular}[c]{@{}c@{}}Coverage\\admissible\end{tabular}
  & \begin{tabular}[c]{@{}c@{}}CApp\\feasible\end{tabular}\\
\midrule
P1: Bayes shrinkage / KT-type predictor          & \checkmark & \checkmark & N/A & N/A & \checkmark\\[4pt]
P2: Plug-in MLE $S_n/n$ (unclipped)              & \checkmark & $\times$   & N/A & N/A & $\times$\\[2pt]
P2$^{\epsilon}$: Clipped plug-in (numerical convention) & $\times^{\dagger}$ & not used & N/A & N/A & $\times^{\dagger}$\\[4pt]
P3: LR e-process                               & \checkmark & N/A        & \checkmark & N/A & N/A\\[4pt]
P4: Conformal prediction set                   & N/A        & N/A        & N/A   & \checkmark & N/A\\
\bottomrule
\end{tabular}

\smallskip
\footnotesize\emph{Note.} On the full Bernoulli model, the Jeffreys
Bayes predictive is the KT closed form $(S_n+\tfrac12)/(n+1)$,
whose CApp boundary-feasibility is established by
Lemma~\ref{lem:df-boundary}.  On compact submodels
$\Theta=[\eta,1-\eta]\subset(0,1)$, the Blackwell-admissible
witness is the restricted full-support Bayes predictive
$\delta^B_n=\Ebb_{\Pi_J}[\theta\mid X_{1:n}]$ under the
renormalized Jeffreys prior, certified by
Corollary~\ref{cor:noshame-compact}; this is generally not the KT
closed form.
\end{table}

\begin{corollary}[Martingale property is not criterion-determining]%
\label{cor:martingale-insufficient}
P1, P2, P3 all satisfy the martingale property and belong to
different admissibility classes.  The martingale property does not
determine membership in $\mathfrak{B}$, $\mathfrak{A}$, or
$\mathfrak{C}$.
\end{corollary}

\begin{corollary}[No universal certificate]\label{cor:no-universal}
No one of the four admissibility certificates
(Blackwell supporting-hyperplane, anytime-valid martingale,
fixed-level coverage frontier in expected length, Ces\`{a}ro
steering) implies all the others.  A modular predictive system in $\Sigma$ may satisfy
several of the four criteria componentwise (examples are recorded
in Section~\ref{sec:compat-size}), but joint admissibility must be
specified one coordinate at a time, and cannot be obtained from a
single one of the four orders.
\end{corollary}

\subsection{Extended procedures}\label{sec:extended-procedures}

The extended proof of Theorem~\ref{thm:extended-separation} uses
two additional ingredients: the Krichevsky--Trofimov forecaster
$\hat p_n^{\mathrm{KT}}=(S_n+\tfrac12)/(n+1)$ as a CApp
boundary-feasible strategy under log loss
(Lemma~\ref{lem:df-boundary}), and a misspecified-prior Bayes rule
$\delta_\Pi\equiv 0.3$ on the compact Bernoulli model
$\Theta=[\eta,1-\eta]\ni 0.3, 0.7$ as a
Blackwell-admissible predictor whose Ces\`{a}ro-average risk at the
true parameter never reaches the entropy floor.  Foster--Vohra
calibration \citep{foster1998} is discussed only as an
interpretive example under a calibration-error objective, not as a
log-loss CApp witness (Remark~\ref{rem:df-vs-logloss}).

\paragraph{Beyond Bernoulli.}
The Bernoulli laboratory is deliberately minimal.  Analogous
witnesses may be available in Gaussian and nonparametric models,
but their admissibility and complete-class assumptions require
separate verification.

\section{Numerical Illustrations}\label{sec:numerical}

The following three Monte~Carlo experiments illustrate the theorem
in finite samples; all use $B=10{,}000$ replications.

\subsection{Bayes vs.\ plug-in under log loss}

We draw $X_1,\ldots,X_n\overset{\mathrm{iid}}{\sim}\mathrm{Bern}(0.3)$
and compare the next-step log loss
$L(\theta,p)=-\theta\log p-(1-\theta)\log(1-p)$ for the Bayes
predictive $\hat p_n^B=(S_n+\tfrac12)/(n+1)$ under
$\mathrm{Beta}(\tfrac12,\tfrac12)$ and the plug-in MLE
$\hat p_n^{\mathrm{pi}}=S_n/n$.  Theorem~\ref{thm:insuff}(iii)
states that the unclipped MLE has \emph{infinite extended-real
risk for every finite $n$}, because
$P_\theta(S_n\in\{0,n\})=(1-\theta)^n+\theta^n$ is strictly
positive for every $\theta\in(0,1)$.  The simulation therefore
reports two distinct quantities (Table~\ref{tab:demo1}): a
numerically clipped MLE risk
(bounding the MLE away from the simplex boundary by
$\epsilon=10^{-6}$), and the \emph{theoretical} unclipped risk,
which is $+\infty$ for every $n$.  The clipped column is a
\emph{numerical convention} only, included for visual comparison
with the Bayes column.

\begin{table}[H]
\centering\small
\caption{Average log loss in nats ($\theta=0.3$, $B=10{,}000$).
The unclipped MLE's log loss is $+\infty$ on every realization where
$S_n\in\{0,n\}$, an event with strictly positive probability
$(1-\theta)^n+\theta^n>0$ for every finite $n$ and every
$\theta\in(0,1)$; the theoretical unclipped MLE risk is therefore
$+\infty$ uniformly in $n$.  The clipped column applies the
convention $\hat p_n^{\mathrm{pi}}\mapsto
\max(\epsilon,\min(1-\epsilon,\hat p_n^{\mathrm{pi}}))$ with
$\epsilon=10^{-6}$ before computing log loss; this is a numerical
convention only and should not be interpreted as the true risk.
The final column reports the exact theoretical boundary
probability
$P_\theta(S_n\in\{0,n\})=0.7^n+0.3^n$.}
\label{tab:demo1}
\begin{tabular}{@{}rcccc@{}}
\toprule
$n$ & Bayes risk & MLE clipped ($\epsilon=10^{-6}$) & MLE unclipped (theoretical) & $P_\theta\{S_n\in\{0,n\}\}$\\
\midrule
  5 & 0.694 & 1.286 & $+\infty$ & $1.70\times 10^{-1}$\\
 10 & 0.658 & 0.757 & $+\infty$ & $2.83\times 10^{-2}$\\
 25 & 0.630 & 0.633 & $+\infty$ & $1.34\times 10^{-4}$\\
 50 & 0.621 & 0.621 & $+\infty$ & $1.80\times 10^{-8}$\\
100 & 0.616 & 0.616 & $+\infty$ & $3.23\times 10^{-16}$\\
\bottomrule
\end{tabular}
\end{table}

The boundary probability shrinks rapidly with $n$ but is never
zero: $(0.7)^{100}+(0.3)^{100}\approx 3.23\times 10^{-16}$ is
positive though far below any Monte Carlo resolution.  The
clipped MLE numerically converges to the Bayes risk because the
boundary event becomes unobservable in samples of size $B$, but
the \emph{theoretical} risk of the unclipped MLE remains $+\infty$
for every $n$.  The Bayes rule, by contrast, has finite risk
uniformly across all realizations and is never dominated for any
$n\ge 1$, consistent with Theorem~\ref{thm:insuff} and the
risk-set geometry of Section~\ref{sec:geometry}.

\subsection{Anytime validity: e-process vs.\ naive peeking}

We draw $X_1,\ldots,X_{200}\overset{\mathrm{iid}}{\sim}\mathrm{Bern}(0.5)$
under $H_0:\theta=0.5$ and compare two sequential testing strategies
at nominal level $\alpha=0.05$.  The \emph{e-process} accumulates a
running likelihood ratio using the Bayes predictive as the
alternative plug-in and rejects whenever $E_t\ge 1/\alpha$; it is
an element of $\cC_{\mathrm{AV}}$ and controls type-I error at
every stopping time (Theorem~\ref{thm:ramdas}).  The \emph{naive
strategy} performs a $z$-test at each of five pre-specified sample
sizes $n\in\{10,20,50,100,200\}$ and rejects if any test exceeds
$z_{0.025}=1.96$.  Table~\ref{tab:demo2} records the rejection
rates.  The e-process remains below $\alpha$ while naive peeking
inflates the type-I error to $0.165$, showing that
$\cC_{\mathrm{AV}}$-feasibility is a binding constraint that the
unrestricted testing class violates.

\begin{table}[H]
\centering\small
\caption{Type-I error under $H_0:\theta=0.5$ ($B=10{,}000$,
  $\alpha=0.05$).}
\label{tab:demo2}
\begin{tabular}{@{}lc@{}}
\toprule
Strategy & Rejection rate\\
\midrule
E-process (anytime-valid) & 0.031\\
Naive peeking ($5$ looks) & 0.165\\
\midrule
Nominal $\alpha$          & 0.050\\
\bottomrule
\end{tabular}
\end{table}

\subsection{Conformal coverage under covariate shift}

We set $Y\mid X=x\sim N(0,(1+x)^2)$ and construct split-conformal
prediction intervals using the naive score $s(x,y)=|y|$ with
$n_{\mathrm{cal}}=500$ calibration points and
$n_{\mathrm{test}}=2{,}000$ test points at level $1-\alpha=0.90$.
Table~\ref{tab:demo3} compares three scenarios: (A)~calibration and
test both under $X\sim\mathrm{Uniform}[0,1]$; (B)~calibration
under $\mathrm{Uniform}[0,1]$, test under
$X\sim\mathrm{Beta}(2,5)$; (C)~both under $\mathrm{Beta}(2,5)$.
When calibration and test distributions match, marginal coverage
holds at the nominal level.  Under covariate shift (Scenario~B),
the $\mathrm{Uniform}$-calibrated quantile is too wide for the
$\mathrm{Beta}(2,5)$ test population (which concentrates $X$ near
zero, where the conditional variance is smaller), inflating coverage
to $0.946$.  Re-calibrating under the test distribution
(Scenario~C) restores nominal coverage and yields a tighter interval.
The calibration quantile itself shifts from $2.51$ to $2.13$,
illustrating that marginal coverage validity is a property of the
exchangeable joint distribution, not a universal guarantee across
arbitrary design points, consistent with
Theorem~\ref{thm:conformal-impossibility} and the scope of
the marginal-coverage guarantee in Definition~\ref{def:coverage}.  Scenario~B
breaks exchangeability, so the conformal guarantee no longer
applies; the observed overcoverage reflects distributional
mismatch, not a violation or strengthening of the theorem.

\begin{table}[H]
\centering\small
\caption{Split-conformal coverage ($1-\alpha=0.90$,
  $n_{\mathrm{cal}}=500$, $n_{\mathrm{test}}=2{,}000$).}
\label{tab:demo3}
\begin{tabular}{@{}llccc@{}}
\toprule
& Calibration $\to$ Test & Quantile & Coverage & Half-width\\
\midrule
A & $\mathrm{Unif}\to\mathrm{Unif}$ & 2.51 & 0.900 & 2.51\\
B & $\mathrm{Unif}\to\mathrm{Beta}(2,5)$ & 2.51 & 0.946 & 2.51\\
C & $\mathrm{Beta}(2,5)\to\mathrm{Beta}(2,5)$ & 2.13 & 0.900 & 2.13\\
\bottomrule
\end{tabular}
\end{table}

\section{Motivations and Interpretations}\label{sec:applications}

This section relates the criterion-relative framework to
forecasting, sequential testing, and conformal wrappers; it
introduces no new theorems.

\paragraph{How to use this paper.}
To analyze a modular predictive system, specify
$(\cD_C,\cF_C,\cR_C)$ separately for each component and attach each
guarantee only to the coordinate for which it was proved
(Definition~\ref{def:cb} and
Figure~\ref{fig:applications-flowchart}).  A system may satisfy
several criteria, but none follows automatically from another.

\subsection{Probabilistic forecasting and calibration}%
\label{sec:app-forecasting}

Empirical calibration compares realized frequencies with announced
probabilities, typically within forecast strata
\citep{dawid1982,foster1998}.  Property~C3 of
Definition~\ref{def:three-coherence} is related but distinct: it is
measure-relative self-consistency under the forecaster's own
predictive law.  Neither notion is identified with the other here.
The Bernoulli plug-in counterexample (Theorem~\ref{thm:insuff}) shows
that C3 does not imply Blackwell admissibility under log loss, and
empirical calibration alone likewise gives no log-loss
non-dominance guarantee.  If calibration is required, it can be
stated as a feasibility constraint before Bayesian log loss is
optimized within the declared feasible class.

\subsection{Sequential testing and safe monitoring}%
\label{sec:app-clinical}

AV-admissibility concerns type-I error control at every
data-dependent stopping time, whereas Neyman--Pearson optimality
concerns power at a fixed sample size; they are different criteria.
Within AV, the
constrained Bayes schema selects a prior or mixture to optimize
expected growth and power inside $\cC_{\mathrm{AV}}$, recovering
the growth-rate optimal $e$-value (GROW) construction of
\citet{grunwald2024} as an instance of the schema.  Fixed-sample
power and anytime-valid safety remain distinct criteria.

\subsection{Conformal wrappers for black-box predictors}%
\label{sec:app-conformal}

A conformal wrapper produces a predictive system $\Delta\in\Sigma$
that lies in $\mathfrak{C}$ because its set coordinate
$\hat C^{\mathrm{conf}}$ is expected-length efficient within its
rank-indexed family in $\cF_{\mathrm{Cov}}(\alpha)$
(Proposition~\ref{prop:conf-admissible}), while the predictor
coordinate may or may not lie in $\mathfrak{B}$.  A conformal
wrapper adds a coverage guarantee to the system but does not
certify the base predictor as Blackwell admissible; conversely, a
Bayes or log-loss-optimal predictor does not supply marginal
coverage.  See \citet{angelopoulos2023} for a survey of conformal
wrappers.  The conformal impossibility result
of \citet{foygel2021}
(Theorem~\ref{thm:conformal-impossibility}) remains the binding
constraint on conditional coverage.

\begin{figure}[ht]
\centering
\resizebox{0.92\linewidth}{!}{%
\begin{tikzpicture}[>=Stealth, scale=1.0,
  appbox/.style={draw, thick, rounded corners=5pt,
    minimum width=3.6cm, minimum height=1.8cm, align=center,
    font=\footnotesize},
  cbbox/.style={draw, thick, rounded corners=4pt, fill=gray!10,
    minimum width=7.0cm, minimum height=1.4cm, align=center,
    font=\small},
  arr/.style={->, thick, black!70}
]
  \node[cbbox] (CB) at (8.0,3.5)
    {\textbf{Constrained Bayes schema}\\[2pt]
     $\min\limits_{\delta\in\cF_C\subseteq\cD_C}
       \int_\Theta \cR_C(\theta,\delta)\,\dx\Pi(\theta)$};

  \node[appbox, fill=blue!8] (F) at (0.6,0)
    {Forecasting / calibration\\[3pt]
     \scriptsize $\cD_C=\cD,\
     \cF_C=\{\delta:\mathrm{calibrated}\}$\\[1pt]
     \scriptsize Objective: log loss};
  \node[appbox, fill=green!8] (T) at (5.4,0)
    {Sequential monitoring\\[3pt]
     \scriptsize $\cD_C=\cF_C=\cC_{\mathrm{AV}}$\\[1pt]
     \scriptsize Certificate: point-null AV martingale};
  \node[appbox, fill=orange!8] (C) at (10.2,0)
    {Conformal wrapper\\[3pt]
     \scriptsize $\cF_C=\cF_{\mathrm{Cov}}(\alpha)\cap\cG\subseteq\cD_C$\\[1pt]
     \scriptsize Certificate: coverage};
  \node[appbox, fill=purple!8] (D) at (15.0,0)
    {Online strategy\\[3pt]
     \scriptsize $\cF_C=\cF_{\mathrm{App}}\subseteq\cD_C$\\[1pt]
     \scriptsize Example: KT under log loss};

  \draw[arr] (CB.south) -- (F.north);
  \draw[arr] (CB.south) -- (T.north);
  \draw[arr] (CB.south) -- (C.north);
  \draw[arr] (CB.south) -- (D.north);

  \node[font=\footnotesize\itshape, text=black!65] at (7.8,-1.4)
    {Specify $(\cD_C,\cF_C,\cR_C)$ first,
     then optimize Bayesian risk within $\cF_C$.};
\end{tikzpicture}%
}
\caption{The constrained Bayes template in four domains.  Each
  criterion has its own decision space $\cD_C$, feasible set
  $\cF_C\subseteq\cD_C$, and risk functional $\cR_C$; Bayesian
  integrated risk is optimized within $\cF_C$
  (Definition~\ref{def:cb}).}
\label{fig:applications-flowchart}
\end{figure}

\section{Discussion}\label{sec:conclusion}

The separation theorems show that none of the four declared
criterion classes subsumes another.  An external product or
meta-order can be imposed, but it adds a choice not supplied by the
four criteria.  The Bernoulli laboratory gives explicit witnesses
for each ordered pair.  Within the Blackwell geometry,
Proposition~\ref{prop:boundary} and Theorem~\ref{thm:supporting}
identify finite-valued no-shame risk vectors with lower-boundary
points supported by a prior; Corollaries~\ref{cor:noshame}
and~\ref{cor:noshame-compact} give full-support conditions for
finite and compact $\Theta$, respectively.

Martingale coherence links single-prior Bayes prediction with
point-null anytime-valid admissibility, but the link is
measure-relative and does not extend to coverage or CApp.  The
plug-in example shows that self-consistency under the forecaster's
own predictive distribution is not the same as non-dominance under
the data-generating process.

The constrained Bayes formulation (Definition~\ref{def:cb}) is a
common design template, not a common ordering: each criterion
retains its own decision space, feasible set, and risk functional.
When the restricted risk set is finite-dimensional, closed,
convex, and finite-valued, $\Pi$ is a normalized normal to a
supporting hyperplane of that criterion's restricted risk frontier
(Remark~\ref{rem:prior-sensitivity}).  Under log loss the CApp
witness is the KT forecaster (Lemma~\ref{lem:df-boundary}); under
calibration-error objectives, defensive forecasting
\citep{vovk2005defensive} and Foster--Vohra \citep{foster1998}
provide fixed-point constructions.

Several directions remain open.  Extending the separation theorem
to nonparametric and infinite-dimensional models would clarify the
scope of the $\Sigma$-space construction.  Understanding how
constrained Bayes interacts with minimax phenomena, especially in
high-dimensional regimes where James--Stein-type effects arise,
would connect the framework to classical complete-class theory.
Finally, modern ML pipelines raise a practical multi-objective
question: how should one choose, combine, and communicate
certificates when calibration, coverage, and sequential validity
are all desired?

\paragraph{Broader-impact remark.}
In deployed predictive systems, practitioners should state which
criterion is being used, why it is appropriate, and which
guarantees are \emph{not} provided.  A system satisfying one
criterion does not inherit the guarantees of the others without an
additional argument.

\bibliographystyle{plainnat}
\bibliography{bayes_no_shame_arxiv_final}

@book{blackwell1954,
  author    = {David Blackwell and Meyer A. Girshick},
  title     = {Theory of Games and Statistical Decisions},
  publisher = {Wiley},
  address   = {New York},
  year      = {1954}
}

@article{blackwell1956,
  author    = {David Blackwell},
  title     = {An Analog of the Minimax Theorem for Vector Payoffs},
  journal   = {Pacific Journal of Mathematics},
  volume    = {6},
  number    = {1},
  pages     = {1--8},
  year      = {1956}
}

@book{wald1950,
  author    = {Abraham Wald},
  title     = {Statistical Decision Functions},
  publisher = {Wiley},
  address   = {New York},
  year      = {1950}
}

@book{rockafellar1970,
  author    = {R. Tyrrell Rockafellar},
  title     = {Convex Analysis},
  publisher = {Princeton University Press},
  address   = {Princeton, NJ},
  year      = {1970}
}

@book{aliprantis2006,
  author    = {Charalambos D. Aliprantis and Kim C. Border},
  title     = {Infinite Dimensional Analysis: A Hitchhiker's Guide},
  edition   = {Third},
  publisher = {Springer},
  address   = {Berlin},
  year      = {2006}
}

@book{cesabianchi2006,
  author    = {Nicol\`{o} Cesa-Bianchi and G\'{a}bor Lugosi},
  title     = {Prediction, Learning, and Games},
  publisher = {Cambridge University Press},
  address   = {Cambridge},
  year      = {2006}
}

@article{definetti1937,
  author    = {Bruno {De Finetti}},
  title     = {La pr\'{e}vision: ses lois logiques, ses sources subjectives},
  journal   = {Annales de l'Institut Henri Poincar\'{e}},
  volume    = {7},
  number    = {1},
  pages     = {1--68},
  year      = {1937}
}

@article{hewitt1955,
  author    = {Edwin Hewitt and Leonard J. Savage},
  title     = {Symmetric Measures on {C}artesian Products},
  journal   = {Transactions of the American Mathematical Society},
  volume    = {80},
  number    = {2},
  pages     = {470--501},
  year      = {1955}
}

@article{fong2024,
  author    = {Edwin Fong and Chris Holmes and Stephen G. Walker},
  title     = {Martingale Posterior Distributions},
  journal   = {Journal of the Royal Statistical Society: Series B},
  volume    = {85},
  number    = {5},
  pages     = {1357--1391},
  year      = {2023}
}

@article{ramdas2022,
  author    = {Aaditya Ramdas and Johannes Ruf and Martin Larsson
               and Wouter Koolen},
  title     = {Admissible Anytime-Valid Sequential Inference Must Rely
               on Nonnegative Martingales},
  journal   = {arXiv preprint arXiv:2009.03167},
  year      = {2022}
}

@article{ramdas2023,
  author    = {Aaditya Ramdas and Peter Gr\"{u}nwald and Vladimir Vovk
               and Glenn Shafer},
  title     = {Game-Theoretic Statistics and Safe Anytime-Valid Inference},
  journal   = {Statistical Science},
  volume    = {38},
  number    = {4},
  pages     = {576--601},
  year      = {2023}
}

@article{grunwald2024,
  author    = {Peter Gr\"{u}nwald and Rianne {de Heide} and Wouter Koolen},
  title     = {Safe Testing},
  journal   = {Journal of the Royal Statistical Society: Series B},
  volume    = {86},
  number    = {5},
  pages     = {1091--1128},
  year      = {2024},
  note      = {Preprint: arXiv:1906.07801}
}

@article{foygel2021,
  author    = {Rina {Foygel Barber} and Emmanuel J. Cand\`{e}s and Aaditya
               Ramdas and Ryan J. Tibshirani},
  title     = {The Limits of Distribution-Free Conditional Predictive Inference},
  journal   = {Information and Inference: A Journal of the IMA},
  volume    = {10},
  number    = {2},
  pages     = {455--482},
  year      = {2021}
}

@book{vovk2005,
  author    = {Vladimir Vovk and Alex Gammerman and Glenn Shafer},
  title     = {Algorithmic Learning in a Random World},
  publisher = {Springer},
  address   = {New York},
  year      = {2005}
}

@inproceedings{vovk2005defensive,
  author    = {Vladimir Vovk and Akimichi Takemura and Glenn Shafer},
  title     = {Defensive Forecasting},
  booktitle = {Proceedings of the Tenth International Workshop on
               Artificial Intelligence and Statistics ({AISTATS})},
  pages     = {365--372},
  year      = {2005}
}

@book{ville1939,
  author    = {Jean Ville},
  title     = {\'{E}tude critique de la notion de collectif},
  publisher = {Gauthier-Villars},
  address   = {Paris},
  year      = {1939}
}

@article{gneiting2007,
  author    = {Tilmann Gneiting and Adrian E. Raftery},
  title     = {Strictly Proper Scoring Rules, Prediction, and Estimation},
  journal   = {Journal of the American Statistical Association},
  volume    = {102},
  number    = {477},
  pages     = {359--378},
  year      = {2007}
}

@article{doob1949,
  author    = {Joseph L. Doob},
  title     = {Application of the Theory of Martingales},
  journal   = {Le Calcul des Probabilit\'{e}s et ses Applications},
  publisher = {CNRS},
  pages     = {23--27},
  year      = {1949}
}

@article{kleijn2012,
  author    = {Bas J. K. Kleijn and Aad W. {van der Vaart}},
  title     = {The {B}ernstein--{v}on {M}ises Theorem Under Misspecification},
  journal   = {Electronic Journal of Statistics},
  volume    = {6},
  pages     = {354--381},
  year      = {2012}
}

@book{shafer2019,
  author    = {Glenn Shafer and Vladimir Vovk},
  title     = {Game-Theoretic Foundations for Probability and Finance},
  publisher = {Wiley},
  address   = {Hoboken, NJ},
  year      = {2019}
}

@inproceedings{abernethy2011,
  author    = {Jacob D. Abernethy and Peter L. Bartlett and Elad Hazan},
  title     = {Blackwell Approachability and No-Regret Learning Are Equivalent},
  booktitle = {Proceedings of the 24th Annual Conference on Learning
               Theory ({COLT})},
  pages     = {27--46},
  year      = {2011}
}

@article{hannan1957,
  author    = {James Hannan},
  title     = {Approximation to {B}ayes Risk in Repeated Play},
  journal   = {Contributions to the Theory of Games},
  volume    = {3},
  pages     = {97--139},
  year      = {1957},
  publisher = {Princeton University Press}
}

@book{berge1963,
  author    = {Claude Berge},
  title     = {Topological Spaces},
  publisher = {Oliver and Boyd},
  address   = {Edinburgh},
  year      = {1963}
}

@article{krichevsky1981,
  author    = {Raphail E. Krichevsky and Victor K. Trofimov},
  title     = {The Performance of Universal Encoding},
  journal   = {IEEE Transactions on Information Theory},
  volume    = {27},
  number    = {2},
  pages     = {199--207},
  year      = {1981}
}

@article{dawid1982,
  author    = {A. Philip Dawid},
  title     = {The Well-Calibrated {B}ayesian},
  journal   = {Journal of the American Statistical Association},
  volume    = {77},
  number    = {379},
  pages     = {605--610},
  year      = {1982}
}

@article{angelopoulos2023,
  author    = {Anastasios N. Angelopoulos and Stephen Bates},
  title     = {Conformal Prediction: A Gentle Introduction},
  journal   = {Foundations and Trends in Machine Learning},
  volume    = {16},
  number    = {4},
  pages     = {494--591},
  year      = {2023}
}

@book{williams1993shame,
  author    = {Bernard Williams},
  title     = {Shame and Necessity},
  publisher = {University of California Press},
  address   = {Berkeley},
  year      = {1993}
}

@article{foster1998,
  author    = {Dean P. Foster and Rakesh V. Vohra},
  title     = {Asymptotic Calibration},
  journal   = {Biometrika},
  volume    = {85},
  number    = {2},
  pages     = {379--390},
  year      = {1998}
}

@article{hart2001,
  author    = {Sergiu Hart and Andreu Mas-Colell},
  title     = {A General Class of Adaptive Strategies},
  journal   = {Journal of Economic Theory},
  volume    = {98},
  number    = {1},
  pages     = {26--54},
  year      = {2001}
}

@article{choe2026filtrations,
  author    = {Yo Joong Choe and Aaditya Ramdas},
  title     = {Combining Evidence Across Filtrations},
  journal   = {Journal of the Royal Statistical Society: Series B
               (Statistical Methodology)},
  year      = {2026},
  pages     = {qkag058},
  doi       = {10.1093/jrsssb/qkag058},
  note      = {Advance access; arXiv:2402.09698}
}

\end{document}